\documentclass{article}

\usepackage{colt10e}

\usepackage{times}
\usepackage{amsfonts}
\usepackage{amsmath}
\usepackage[psamsfonts]{amssymb}
\usepackage{latexsym}
\usepackage{color}
\usepackage{graphics}
\usepackage{enumerate}
\usepackage{amstext}
\usepackage{url}
\usepackage{epsfig}
\usepackage{mlapa}

\usepackage{picins}
\usepackage{wrapfig}

\usepackage{graphicx} 
\usepackage{subfigure}

\newcommand\tstrut{\rule{0pt}{2.4ex}}
\newcommand\bstrut{\rule[-1.0ex]{0pt}{0pt}}

\def\R{\mathbb{R}}
\def\B{\mathbb{B}}
\def\T{\mathbb{T}}

\def\d{\partial}
\def\n{\textrm{\bf{n}}}

\newcommand{\Om}{\overline{\Omega}}
\newcommand{\om}{\Omega}
\newcommand{\E}{\mathbb{E}}

\title{Behavior of Graph Laplacians on Manifolds with Boundary}

\author{Xueyuan Zhou\\
Department of Computer Science
\\ University of Chicago
\\\texttt{\small zhouxy@cs.uchicago.edu}
\And Mikhail Belkin\\
Department of Computer Science and Engineering\\
Ohio State University\\
\texttt{\small mbelkin@cse.ohio-state.edu}}

\begin{document}

\maketitle

\begin{abstract}
In manifold learning, algorithms based on graph Laplacians
constructed from data have received considerable attention both in
practical applications and theoretical analysis. In particular, the
convergence of graph Laplacians obtained from sampled data to
certain continuous operators has become an active research topic
recently. Most of the existing work has been done under the
assumption that the data is sampled from a manifold without boundary
or that the functions of interests are evaluated at a point away
from the boundary. However, the question of boundary behavior is of
considerable practical and theoretical interest. In this paper we
provide an analysis of the behavior of graph Laplacians at a point
near or on the boundary, discuss their convergence rates and their
implications and provide some numerical results. It turns out that
while points near the boundary occupy only a small part of the total
volume of a manifold, the behavior of  graph Laplacian there has
different scaling properties from its behavior elsewhere on the
manifold, with global effects on the whole manifold, an observation
with potentially important implications for the general problem of
learning on manifolds.
\end{abstract}
\section{Introduction}
Graph Laplacian constructed from data points is a key element in
many machine learning algorithms including spectral clustering,
e.g.,~\cite{uvon}, semi-supervised
learning~\cite{zhu2006semi,chapelle2006ssl} and dimensionality
reduction~\cite{BelkinLapMap2003}, as well as a number of other
applications. A large amount of work in recent years has been
centered on analyzing various theoretical aspects of graph
Laplacians on manifolds, and, in particular, on their different
modes of convergence, when the data goes to infinity and/or the
parameters, such as kernel bandwidth, tend to
zero~\cite{belkinThesis,lafon,hein,CoifmanLafon2006,singer,gine,Hein07graphlaplacians,belkin2008,uvon2008,rosasco2010}.
A typical result in that direction shows that the discrete graph
Laplacian converges\footnote{Different modes of convergence are
possible here, such as different types of pointwise or uniform
convergence or convergence of eigenvectors.} to the
Laplacian-Beltrami operator on manifolds when the bandwidth
parameter of the kernel is chosen as an appropriate function of the
number of data points. These results help to clarify our
understanding of the underlying objects, to shed light on properties
of the algorithms and to guide the selection of algorithms in
practical applications.

For example, an analysis of  normalized versus unnormalized
Laplacians in~\cite{uvon2008}) suggests that normalization may be
preferable in practical applications. In another example, the
estimators of several graph Laplacian based semi-supervised learning
algorithms had recently been shown to converge to constant solutions
in the limit of infinite unlabeled points while fixing labeled
points \cite{nadler2009}, suggesting the use of iterated Laplacians
\cite{Zhou2011a}, which indeed shows superior performance in
practice.

The spectral convergence of a graph Laplacian is another important
limit analysis of the graph Laplacian, which links directly to
applications. The empirical spectral convergence of spectral
clustering when the sample size $n$ goes to infinity for a fixed
kernel bandwidth $t$ was studied by \cite{uvon2008}, while the
spectral convergence of a graph Laplacian to the Laplace-Beltrami
operator when the kernel bandwidth $t$ goes to zero as $n$ goes to
infinity is studied in \cite{belkinCLEM}.

However, most previous results on graph Laplacians deal with the
setting where the manifold does not have a boundary or when the
operator is analyzed at a point away from the boundary. Arguably, it
is a significant short-coming of these analyses, since manifolds or
domains with boundary are present explicitly or implicitly in many
problems of significant interest in data analysis. Perhaps the
simplest example is the fact that the pixel intensity of a
gray-scale image cannot be smaller than zero, providing a natural
boundary condition for any image manifolds. A more interesting
example is in motion analysis, where the manifold of configurations
of a human or robot body (perhaps embedded using video images or
data from sensors attached to limbs) has boundaries corresponding to
the limits for the range of motions of each individual joint. More
generally, it is natural to think that boundaries in data are
present whenever the generating process itself is in some way
constrained. It is clear that if such manifolds are to be learned
from data, the boundary behavior cannot be disregarded.

In the current paper we discuss the boundary behavior of graph
Laplacians by analyzing the graph Laplacian convergence at the
boundary. We show that the graph Laplacian at the boundary converges
to a gradient operator in the direction normal to the boundary, when
the bandwidth parameter $t$ is chosen adaptively as a function of
the number of data points. We provide explicit bounds for the
convergence. One of the key results of our analysis is that both the
behavior and  the scaling of the graph Laplacian near the boundary
is quite different from that in the interior of the manifold.
Specifically, for a fixed function $f(x)$ and a small bandwidth
parameter $t$ the (appropriately scaled) graph Laplacian will be
close to the Laplace-Beltrami operator $\Delta f(x)$ on interior
point $x$, while at the boundary the same object will be close to
the normal derivative $\frac{1}{\sqrt t}\partial_{\bf n} f(x)$. We
see that the large values of the graph Laplacian applied to a fixed
function are likely to correspond to the boundary points. Moreover,
the analysis shows that while there are few points near the boundary
of a manifold, their influence on the graph Laplacian is
disproportionately large and cannot be ignored. This suggests that
the boundary has a global effect on the graph Laplacian, a finding
that is confirmed by our numerical experiments provided in the
paper. Viewed in a different way it suggests that for algorithms
when a graph Laplacian is used as a regularizer, as is the case in
many applications, bounding the norm would lead to the suppression
of the large values near the boundary. Thus the minimizer of the
regularization problem (or similarly, the eigenvectors) should
satisfy the Neumann boundary conditions, i.e., be nearly constant in
the direction orthogonal to the boundary, which is confirmed by our
numerical experiments.

In a related line of investigation we find that  the symmetric
normalized graph Laplacian $L^s$ has a different boundary behavior
from the random walk (asymmetric normalized) and unnormalized graph
Laplacians. Unlike those two, for a fixed function $f(x)$, $L^s
f(x)$ converges to $\frac{1}{\sqrt t}[p(x)]^{1/2}\partial_\n
(f(x)/[p(x)]^{1/2})$ for a boundary point $x$, where $p(x)$ is the
probability density function. This does not lead to the Neumann
boundary condition, and seems strange from a practical point of
view.

As a further illustration of the importance of boundary conditions
in learning theory, we explore the boundary effects for a
reproducing kernel in a simple 1-dimensional example. We also
discuss the limit of the graph Laplacian regularizer on manifolds
with boundary, which cannot be taken for granted to be the same as
the limit on $\R^N$ or manifolds without boundary because of the
boundary behavior of graph Laplacians.

Finally we briefly compare the graph Laplacian built from random
samples to the Laplacian on regular grids in numerical PDE's.

\subsection{Problem Setting}
We now proceed with a more technical setting of the problem. Let
$\Om$ be a compact Riemannian submanifold of intrinsic dimension $d$
embedded in $\R^N$, $\om$ the interior of $\Om$, and $\d \om$ the
boundary of $\om$, which we will assume to satisfy the {\em
necessary smoothness conditions}\footnote{ Instead of spending
several pages to describe these smoothness conditions in this paper,
we refer readers to \cite{belkinThesis,lafon,hein} for more
details.}. Given $n$ random samples $X=\{X_1,\cdots, X_n\}$ drawn
i.i.d. from a distribution with a smooth density function $p(x)$ on
$\Om$ such that $0<a\le p(x)\le b<\infty$, we can build a weighted
graph $G(V,E)$ by mapping each sample point $X_i$ to vertex $v_i$
and assigning a weight $w_{ij}$ to edge $e_{ij}$. One typical weight
function is the Gaussian defined as
$w_{ij}=K_t(X_i,X_j)=1/t^{d/2}e^{-\|X_i-X_j\|^2_{\R^N}/t}$, which is
used in this paper. Let the $n\times n$ matrix $W$ be the edge
weight matrix of graph $G$ with $W(i,j)=w_{ij}$, and $D$ be a
diagonal matrix such that $D_{ii}=\sum_{j}w_{ij}$, then the
unnormalized graph Laplacian is defined as matrix $L^u$
\begin{equation}
    L^u=D-W
\end{equation}
There are several ways of normalizing $L^u$. For instance, the most
commonly used two are the asymmetric random walk normalized version
$L^r=D^{-1}L^u=I-D^{-1}W$ and the symmetric normalized version
$L^s=D^{-1/2}L^uD^{-1/2}=I-D^{-1/2}WD^{-1/2}$.

Another useful way of building a graph Laplacian is governed by a
parameter $\alpha$ such that we first normalize $W$ as
$W_\alpha=D^{-\alpha}WD^{-\alpha}$, then define the unnormalized,
random walk and symmetric normalized graph Laplacians as
\begin{equation}
\begin{array}{rl}
    L_\alpha^u= & D_\alpha-W_\alpha\\
    L_\alpha^r= & I-D_\alpha^{-1}W_\alpha\\
    L_\alpha^s= & I-D_\alpha^{-1/2}W_\alpha D_\alpha^{-1/2}
\end{array}
\end{equation}
where $D_\alpha$ is the corresponding diagonal degree matrix for
$W_\alpha$. It is easy to see when $\alpha=0$, these graph
Laplacians become the commonly used ones without the first step
normalization. Therefore, for each value of $\alpha$, there are
three closely connected empirical graph Laplacians.

The limit study of graph Laplacians primarily involves the limits of
two parameters, sample size $n$ and weight function bandwidth $t$.
As $n$ increases, one typically decreases $t$ to let the graph
Laplacian capture progressively a finer local structure.

With a proper rate as a function of
$n$ and $t$, the limit of $L^uf(x)$ for a given smooth function and
fixed $x$ can be  shown to be $\Delta f(x)$ when $\Om$ is a compact
submanifold of $\R^N$ without boundary and $p(x)$ is a uniform
density. This builds a connection between the
discrete graph Laplacian and the  continuous
Laplace-Beltrami operator $\Delta$ on manifolds, which in $\R^d$ can
be written as
\begin{equation}
    \Delta =\sum_{i=1}^d \frac{\d^2}{\d x_i^2}
\end{equation}
This connection is an important step in providing a theoretical
foundation for many graph Laplacian based machine learning
algorithms. For instance, harmonic functions used in \cite{zhu2003}
for semi-supervised learning is in fact a solution of a Laplace
equation, with a ``point boundary condition'' at labeled points.

The limit of $L_\alpha^r$ and its various aspects, including the
finite sample analysis, are studied
in~\cite{belkinThesis,lafon,hein,singer,gine,Hein07graphlaplacians,belkin2008,belkinCLEM}.
The basic result is that the limit of $L_\alpha^r f(x)$ for
$x\in\om$ is (up to a constant )
\begin{equation}
    \frac{1}{t}L_\alpha^r f(x)\stackrel{p}{\to} -\Delta_s f(x)=-\frac{1}{p^s}\textrm{div}[p^s
    \textrm{grad f(x)}] =-[\Delta +\frac{s}{p}\langle
    \nabla p(x), \nabla \rangle]f(x)
\end{equation}
where $\Delta_s$ is the weighted Laplacian and $s=2(1-\alpha)$.
These papers deal with the analysis of graph Laplacians at an
interior point of the manifold and do not deal with boundary
behavior. The exception to that is the analysis
in~\cite{CoifmanLafon2006}, which includes manifolds with boundary,
assuming the Neumann boundary conditions on the space of functions.
Specifically, the Taylor series for the Gaussian convolution
in~\cite[Lemma 9]{CoifmanLafon2006} involves a term containing the
normal gradient at the boundary, which can be reformulated to obtain
the limit for the graph Laplacian on the manifold boundary. However,
there is no explicit discussion of the boundary behavior as well as
its implication for learning in~\cite{CoifmanLafon2006}. Discrete
graph Laplacian is not considered in that work. We believe that
given the popularity of graph Laplacians in machine learning, the
boundary behavior of graph Laplacians deserves a more detailed
study.

In Section~\ref{sec:GraphLaplacianReview}, we state some existing
results on the limit analysis of the graph Laplacian as well as some
necessary preparatory results, which will be useful for our
analysis. Section~\ref{sec:GraphLaplacianBoundaryLimit} contains our
main Theorem~\ref{thm:limit:L}, which states that near the boundary,
the graph Laplacian converge to the normal gradient and shows the
scaling behavior and explicit rates of convergence. We also show how
the scaling changes between the boundary and the interior points of
the manifold. Numerical examples to support our analyses are
provided in Section~\ref{sec:NumericalExample}. Several important
implications of the boundary behavior of the graph Laplacian are
discussed in Section~\ref{sec:implication}.

\section{Technical Preliminaries}
\label{sec:GraphLaplacianReview} In this section, we review the
existing limit analysis of graph Laplacians $L_\alpha^u$,
$L_\alpha^r$ and $L_\alpha^s$ on points  away from the
boundary of a compact submanifold. We also provide some technical
results useful for our analysis in
Section~\ref{sec:GraphLaplacianBoundaryLimit}.

Given an undirected graph representation of the random sample set
$X$ of size $n$, the weight function with parameter $t$ is defined
as
\begin{equation}
    w_{t}(X_i,X_j)=\frac{1}{t^{d/2}}e^{-\frac{\|X_i-X_j\|^2_{\R^N}}{t}}
\end{equation}
Notice that in this Gaussian weight function, the Euclidean distance
should be used, instead of other distance, e.g., the geodesic on
manifolds. It is this critical feature that on one hand makes the
graph Laplacians computationally attractive, on the other hand has
important implications, which will be discussed in the rest this
paper.

Define the corresponding discrete degree function as
\begin{equation}
    d_{t,n}(X_i)=\frac{1}{n}\sum_{j=1}^n w_{t}(X_i,X_j)
\end{equation}
Then we first normalize the weight function to obtain
\begin{equation}
    w_{\alpha,t}(X_i,X_j)=\frac{w_{t}(X_i,X_j)}{[d_{t,n}(X_i)d_{t,n}(X_j)]^\alpha}
\end{equation}
Note that this weight function also depends on the locations of
$X_i$ and $X_j$ other than the Euclidean distance
$\|X_i-X_j\|_{\R^N}$. We use the three subscripts $\alpha, t, n$ to
emphasize the related parameters. The corresponding discrete degree
function is
\begin{equation}
    d_{\alpha,t,n}(X_i)=\frac{1}{n}\sum_{j=1}^n w_{\alpha,t}(X_i,X_j)
    =\frac{1}{n}\sum_{j=1}^n \frac{w_{t}(X_i,X_j)}{[d_{t,n}(X_i)d_{t,n}(X_j)]^\alpha}
\end{equation}
If the weight matrix for $w_t(X_i,X_j)$ is $W_{t,n}$ and the
corresponding degree matrix is $D_{t,n}$, then the normalized weight
matrix is
\begin{equation}
W_{\alpha,t,n}=D_{t,n}^{-\alpha}W_{t,n}D_{t,n}^{-\alpha}
\end{equation}
By finding the corresponding degree matrix $D_{\alpha,t,n}$, the
unnormalized graph Laplacian is
\begin{equation}
    L_{\alpha,t,n}^u=D_{\alpha,t,n}-W_{\alpha,t,n}
\end{equation}
and the other two normalized versions are defined accordingly as
$L_{\alpha,t,n}^r=I-D_{\alpha,t,n}^{-1}W_{\alpha,t,n}$ and
$L_{\alpha,t,n}^s=I-D_{\alpha,t,n}^{-1/2}W_{\alpha,t,n}D_{\alpha,t,n}^{-1/2}$.

For a fixed smooth function $f(x)$, and any $x\in \Om$ (including
the samples and unseen points), define $L_{\alpha,t,n}^u f(x)$ as
the following,
\begin{equation}
    L_{\alpha,t,n}^u f(x)=\frac{1}{n}\sum_{j=1}^n w_{\alpha,t,n}(x,X_j)(f(x)-f(X_j))
\end{equation}
and similarly for the random walk normalized graph Laplacian
\begin{equation}
    L_{\alpha,t,n}^r f(x)=\frac{\frac{1}{n}\sum_{j=1}^n w_{\alpha,t,n}(x,X_j)(f(x)-f(X_j))}{d_{\alpha,t,n}(x)}
    =f(x)-\frac{1}{n}\sum_{j=1}^n
    \frac{w_{\alpha,t,n}(x,X_j)}{d_{\alpha,t,n}(x)} f(X_j)
\end{equation}
For $L_{\alpha,t,n}^s$, it can be shown that $L_{\alpha,t,n}^s f(x)=
D_{\alpha,t,n}^{-1/2}L_{\alpha,t,n}^u F(x)$ where
$F(x)=D_{\alpha,t,n}^{-1/2}f(x)$. Similar notions also apply to the
degree functions. The intuition is that we treat vector
$(f(X_1),\cdots,f(X_n))^T$ as a sampled continuous function $f(x)$
on $\Om$. As $n\to \infty$, the vector becomes ``closer and closer''
to $f(x)$.

Three useful convergence results for the interior points will be
needed in our analysis~\cite{Hein07graphlaplacians}:
\begin{equation}
    d_{t,n}(x) \stackrel{\rm a.s.}{\longrightarrow} \ C_1 p(x)
\end{equation}
where $C_1=\int_{\R^d}K(\|u\|^2)du$, and
\begin{equation}
    d_{\alpha,t,n}(x) \stackrel{\rm a.s.}{\longrightarrow} C_1^{1-2\alpha}
    [p(x)]^{1-2\alpha}
\end{equation}
The following limit shows that the graph Laplacian on points that
are away from the boundary converge to the density weighted
Laplace-Beltrami operator with a proper rate of $n$ and $t$.
\begin{equation}
\frac{1}{t} L_{\alpha,t,n}^rf(x) \stackrel{\rm a.s.}
{\longrightarrow} -\frac{C_2}{2C_1}\Delta_s f(x)
\end{equation}
where $C_2=\int_{\R^d}K(\|u\|^2)u_1^2du$. The limits of
$L_{\alpha,t,n}^uf(x)$ and $L_{\alpha,t,n}^sf(x)$ can be found in
\cite{Hein07graphlaplacians}

On a $d$-dimensional smooth manifold $\Om$, for an interior point
$x$, the small neighborhood around $x$ is locally equivalent to
whole space $\R^d$, while for a point on the boundary of $\Om$, i.e.
$x\in \d\Omega$, the small neighborhood around $x$ is locally mapped
into a {\em half space} defined as $\R^d_{+}=\{x\in \R^d,x_1\ge
0\}$. This is a key fact that will be used in this paper.

Next we will need a concentration inequality for the finite sample
analysis of the graph Laplacian.
\begin{lemma}(McDiarmid's inequality)\label{lemma:McDiarmid}
Let $X_1,\cdots,X_n$, $\hat{X}_i$ be i.i.d. random variables of
$\R^N$ from density $p(x)\in C^\infty(\Om)$, $0<a\le p(x)\le b
<\infty$, $|f|<M$ and $f$ satisfies
\begin{equation}
    \sup_{X_1,\cdots,X_n,\hat{X}_i}
    |f(X_1,\cdots,X_i,\cdots,X_n)-f(X_1,\cdots,\hat{X}_i,\cdots,X_n)|\le
    c_i, \quad \textrm{ for }1\le i\le n
\end{equation}
then
\begin{equation}
    P(|f(X_1,\cdots,X_n)-\E[f(X_1,\cdots,X_n)]|>\epsilon)\le 2\exp{(-\frac{2\epsilon^2}
    {\sum_{i=1}^n c_i^2})}
\end{equation}
\end{lemma}

\section{Analysis of Graph Laplacian Near Manifold Boundary}\label{sec:GraphLaplacianBoundaryLimit}

In this section, we analyze the limits of the Laplacians
$L_{\alpha,t,n}^rf(x)$, $L_{\alpha,t,n}^uf(x)$ and
$L_{\alpha,t,n}^sf(x)$ when $x$ is on or near the boundary of
manifold $\Om$. The argument roughly follows the lines of the
convergence arguments in~\cite{belkinThesis,hein,CoifmanLafon2006}.

To fix the notation,  in the rest of this paper, we use expressions
without subscript $n$ to indicate the corresponding limit as $n\to
\infty$, and expressions without subscript $t$ to for the limits as
$t\to 0$.
\begin{equation}
\begin{array}{rl}
    w_{t}(x,y)=&K_t(x,y)=\frac{1}{t^{d/2}}K(x,y)\\
                   &=\frac{1}{t^{d/2}}K(\frac{\|x-y\|_{\R^N}^2}{t})=\frac{1}{t^{d/2}}e^{-\frac{\|x-y\|^2_{\R^N}}{t}}\\
    \\
    d_{t}(x)=&\int_{\Om} w_t(x,y) p(y) dy\\
    \\
    w_{\alpha,t}(x,y)=& \frac{w_t(x,y)}{ [d_t(x)d_t(y)]^\alpha}\\
    \\
    d_{\alpha, t}(x)=&\int_{\Om} \frac{w_t(x,y)}{[d_t(x)d_t(y)]^\alpha} p(y) dy\\
\end{array}
\end{equation}
For smooth $f(x)$ and $p(x)$,
\begin{equation}
    L_{\alpha,t}^u f(x)=\int_{\Om} w_{\alpha,
    t}(x,y)(f(x)-f(y))p(y)dy = d_{\alpha,t}(x) L_{\alpha,t}^rf(x)
\end{equation}
and
\begin{equation}
    L_{\alpha,t}^rf(x)=
    f(x)- \int_{\Om}
    \frac{w_{\alpha, t}(x,y)}{d_{\alpha,t}(x)}f(y)p(y)dy
\end{equation}
Similarly, $L_{\alpha,t}^s f(x)$ can be rewritten as $L_{\alpha,t}^s
f(x)=d_{\alpha,t}^{-1/2}(x) L_{\alpha,t}^u F(x)$ with
$F(x)=d_{\alpha,t}^{-1/2}(x) f(x)$.

Next we show the limits of the graph Laplacians on boundary point
$x\in \d\Omega$ as $t\to 0$ and $n\to \infty$ at a proper rate, when
$\Om$ has a {\em smooth} boundary.
\begin{theorem}
\label{thm:limit:L} Let $f\in C^3(\Om)$, $|f(x)|\le M$, $p(x)\in
C^\infty(\Om)$, $0<a\le p(x)\le b<\infty$, $\d \Omega$ be a smooth
boundary of $\Omega$, $x\in \d\Omega$, and $t$ be sufficiently
small, then for the unnormalized graph Laplacian $L_{\alpha,t,n}^u$
\begin{equation}
P(|\frac{1}{\sqrt{t}}
L_{\alpha,t,n}^uf(x)-[-\frac{C_4}{C_3^{2\alpha}}[p(x)]^{1-2\alpha}\d_\n
f(x)]|\ge \epsilon)\le 2 \exp{(-\frac{nt^{d+1}\epsilon^2}{C_0})}
\end{equation}
for the random walk normalized graph Laplacian $L_{\alpha,t,n}^r$
\begin{equation}
P(|\frac{1}{\sqrt{t}} L_{\alpha,t,n}^rf(x)-[-\frac{C_4}{C_3}\d_\n
f(x)]|\ge \epsilon)\le 2 \exp{(-\frac{nt^{d+1}\epsilon^2}{C_0})}
\end{equation}
and for the symmetric normalized graph Laplacian $L_{\alpha,t,n}^s$
\begin{equation}
P(|\frac{1}{\sqrt{t}}
L_{\alpha,t,n}^sf(x)-[-\frac{C_4}{C_3}[p(x)]^{1/2-\alpha}\d_\n
(\frac{f(x)}{[p(x)]^{1/2-\alpha}})]|\ge \epsilon)\le 2
\exp{(-\frac{nt^{d+1}\epsilon^2}{C_0})}
\end{equation}
where $s=2(1-\alpha)$, {\bf n} is inward normal direction, $C_0$
only depends on $M,a,b$ and $\alpha$, $C_3=1/2 \int_{\R^d}
K(\|u\|^2) du$, and $C_4=\int_{\R_{+}^d} K(\|u\|^2)u_1 du$.
\end{theorem}
\begin{proof}
We first show the limit of the expectation of $L_{\alpha,t}^uf(x)$
as $t\to 0$ in step 1 to 3. Then the limit of $L_{\alpha,t}^rf(x)$
and $L_{\alpha,t,n}^sf(x)$ can easily be found with the help of the
limit of discrete degree function $d_{\alpha,t}(x)$. At last, we
obtain the finite sample results by applying Lemma
(\ref{lemma:McDiarmid}).

For a sufficiently small $t$, let $\Omega_1$ be the set of points
that are within distance $O(\sqrt{t})$ from the boundary $\d\Omega$
(a thin layer of ``shell''), and $\Omega_0={\Om}/\Omega_1$. We first
show that for a small $t$, $L_{\alpha,t}^uf(x)$ is approximated by
two different terms on $\Omega_0$ and $\Omega_1$, and more
importantly these two terms have different orders of $t$. Then
together with the limit of $d_{\alpha,t}(x)$, we can find the limit
of $L_{\alpha,t}^rf(x)$ and $L_{\alpha,t}^sf(x)$.

Step 1: The key step for the limit analysis of graph Laplacians is
the approximation on the manifold. Consider
\begin{equation}\label{equ:aimIntegral}
\begin{array}{rl}
    L_{\alpha,t}^uf(x)=&\int_{\Om}
    \frac{K_t(x,y)}{[d_t(x)d_t(y)]^\alpha}(f(x)-f(y))p(y)dy\\
    \\
    =&[d_t(x)]^{-\alpha}\int_{\Om}K_t(x,y)[d_t(y)]^{-\alpha}(f(x)-f(y))p(y)dy\\
\end{array}
\end{equation}
This integral is on the manifold $\Om$. In order to study the the
limit of this integral when $t\to 0$, we can approximate the
integral on an unknown smooth manifold by an integral on its tangent
space at each point $x$ when $t$ is small such that the
approximation errors of each step are comparable. For $x\in \Omega$,
the tangent space is the whole space $\R^d$, while for $x\in
\d\Omega$, the tangent space is the half space $\R_{+}^d$ ($x_1\ge
0$).

When $y\in \Om$ is within an Euclidean ball of radius $O(t^{1/2})$
centered at $x$, in the local coordinate around a fixed $x$, the
origin is point $x$, and let $s=(s_1,\cdots, s_d)$ be the local
geodesic coordinate of $y$, $u=(u_1,\cdots, u_d)$ be the projection
of $y$ on the tangent space at $x$. Then we have the following
important approximation (see \cite[Chapter 4.2]{belkinThesis} and
\cite[Appendix B]{CoifmanLafon2006}).
\begin{equation}\label{equ:basicApprox}
\begin{array}{rl}
    s_i=&u_i + O(t^{3/2})\\
    \\
    \|x-y\|_{\R^N}^2=&\|u\|_{\R^d}^2+O(t^{2})\\
    \\
    \textrm{det}(\frac{dy}{du}) = &1 + O(t)\\
\end{array}
\end{equation}

Step 2: Now we are ready to approximate each of the five terms in
integral (\ref{equ:aimIntegral}) when the integral is taken inside a
ball centered at $x$ having radius $O(t^{1/2})$ in
$\|\cdot\|_{\R^N}$ norm. Notice that $\|u\|_{\R^d}\sim O(t^{1/2})$.
\begin{equation}\label{equ:detailedApprox}
\begin{array}{rl}
    K(\frac{\|x-y\|^2_{\R^N}}{t})= & K(\frac{\|u\|^2_{\R^d}}{t})+O(t^2)\\
    \\
    d_t^{-\alpha}(y) = & d_t^{-\alpha}(x) - \alpha d_t^{-\alpha-1}(x) s^T \nabla d_t(x) + O(s^2)\\
    \\
    =&d_t^{-\alpha}(x) - \alpha d_t^{-\alpha-1}(x) u^T \nabla d_t(x)
    + O(t)\\
    \\
    f(x)-f(y)= & -s^T\nabla f(x)-\frac{1}{2}s^TH(x)s + O(s^3)\\
    \\
    =&-u^T \nabla f(x) - \frac{1}{2}u^TH(x)u + O(t^{3/2})\\
    \\
    p(y)=&p(x)+s^T\nabla p(x) + O(s^2)\\
    \\
    =&p(x)+u^T\nabla p(x)+O(t)\\
\end{array}
\end{equation}
where $H(x)$ is the Hessian of $f(x)$ at point $x$. Notice that, the
order inside of the big oh is determined by the larger one between
the approximation error of $u$ to $s$ which is $O(t^{3/2})$, and the
Taylor expansion error on manifold as a function of $s$. The other
observation is that, the order of the product of these terms is
determined by the third term ($f(x)-f(y)$), the highest order of
which is $O(t^{1/2})$, with the next ones as $O(t)$ and
$O(t^{3/2})$. This means it is enough to keep the approximation
terms up to order $t^{1/2}$.

Combing all the approximation together in a ball of radius
$O(t^{1/2})$ around $x$, with a change of variable $u\to t^{1/2}u$,
we can obtain $L_{\alpha,t}^uf(x)$
\begin{equation}
\begin{array}{rl}
    L_{\alpha,t}^uf(x)=&\int_{\Om}
    \frac{K_t(x,y)}{[d_t(x)d_t(y)]^\alpha}(f(x)-f(y))p(y)dy\\
    \\
    =&\int_{\Om \cap \B_1(x) }
    \frac{K_t(x,y)}{[d_t(x)d_t(y)]^\alpha}(f(x)-f(y))p(y)dy+ O(t^{3/2})\\
    \\
    =&-\frac{1}{t^{d/2}d_t^\alpha(x)}\int_{\Om \cap \B_2(x)} K(\|u\|^2_{\R^d})[(\frac{1}{d_t^\alpha(x)}-
    \sqrt{t} \frac{\alpha u^T \nabla d_t(x)}{(d_t(x))^{\alpha+1}})
    (\sqrt{t}u^T \nabla f(x)+\frac{t}{2} u^T H(x) u )\\
    &  \qquad \qquad \qquad \qquad\ \ (p(x)+\sqrt{t}u^T\nabla p(x))]t^{d/2}du + O(t^{3/2})\\
    \\
    =&-\frac{1}{d_t^\alpha(x)}\int_{\T(x)} K(\|u\|^2_{\R^d})\{\sqrt{t}[\frac{p(x)}{d_t^\alpha(x)}(u^T\nabla f(x))]+\\
    & \qquad \qquad \qquad \qquad \quad t[\frac{u^T \nabla f(x)\times u^T \nabla
    p(x)}{d_t^\alpha(x)}-
    \alpha\frac{p(x)u^T\nabla d_t(x)\times u^T \nabla f(x)}
    {d_t^{\alpha+1}(x)}+\frac{1}{2}\frac{p(x)}{d_t^\alpha(x)}u^TH(x)u]\} du\\
    &  \qquad \qquad \qquad \qquad \ \ +O(t^{3/2})
\end{array}
\end{equation}
where $\B_1(x)$ is a ball of radius $O(t^{1/2})$ in
$\|\cdot\|_{\R^N}$ norm centered at $x$, while $\B_2(x)$ is a ball
of radius $O(t^{1/2})$ in $\|\cdot\|_{\R^d}$ norm, and $\T(x)$ is
the tangent space at point $x$. For a sufficiently small $t$, the
first step replaces the integral over the whole $\Om$ with ball
$\B_1(x)$, generating an error $O(t^{3/2})$ \cite[Appendix
B]{CoifmanLafon2006}. Then this integral is the same as the integral
over a ball on the manifold $\Om$. Finally, for an interior point
$x$, $\T(x)=\R^d$, which means function $K(\|u\|^2_{\R^d})$ is a
even function of $u$. When taking the integral, the first term which
has order $\sqrt{t}$ is odd and therefore vanishes. Then the three
left terms that are of order $t$ inside the integral are exactly the
weighted Laplacian at $x$, which is of order $t$. For a boundary
point $x$, $\T(x)=\R_+^d$. Next we study the interior points.

\begin{figure}[h]
\centering \epsfig{file=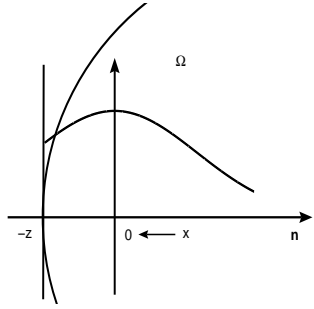, height=4cm, width=5.5cm}
\caption{Gaussian weight at $x$ near the
boundary.}\label{fig:bd:integral}
\end{figure}

Step 3: In Figure.~\ref{fig:bd:integral}, $x\in \Omega_1$ (the
``shell'') is a point near the boundary, $\n$ is the inward normal
direction, and $-z$ is the nearest boundary point to $x$ along $\n$.
In the local coordinate system, $x$ is the origin, and along the
normal direction the Gaussian convolution is from $-z$ to $+\infty$,
which is not symmetric. Therefore, $K(\|u\|^2_{\R^d})$ is not an
even function in the normal direction, so the highest order term is
the order $O(\sqrt{t})$ term.

In this case, all the odd terms of $u_i$ still will vanish in all
directions except the normal direction $\n$, and the most important
point is that the leading term along the normal direction is of
order $\sqrt{t}$, while for interior points it is $t$. Next we
assume $u_1$ is the normal direction.
\begin{equation}
    \frac{1}{\sqrt{t}}L_{\alpha,t}^uf(x)=-\frac{1}{d_t^{2\alpha}(x)}p(x)\d_\n f(x)\int_{-\infty}^{+\infty} \cdots \int_{-\infty}^{+\infty}\int_{-z}^\infty
    K(\|u\|^2_{\R^d})u_1 du_1du_2\cdots du_d + O(\sqrt{t})
\end{equation}
where $z$ is the distance to the nearest point of $x$ on the
boundary $\d\Omega$ along the normal direction ($z\ge 0$) as shown
in Figure \ref{fig:bd:integral}. When $t\to 0$, $z\to 0$ in the
local coordinate system
\begin{equation}
    \lim_{t\to 0}\frac{1}{\sqrt{t}}L_{\alpha,t}^uf(x)
    =-\frac{C_4}{C_3^{2\alpha}} [p(x)]^{1-2\alpha} \d_\n f(x)
\end{equation}
where $C_3=\int_{\R^d_{+}}K(\|u\|^2)du=1/2 C_1$,
$C_4=\int_{\R_{+}^d} K(\|u\|^2)u_1 du$. This result also needs the
following limits, which generalize \cite[Proposition 2.33]{hein} to
points on the boundary.

\begin{equation}
    \lim_{t\to 0} d_{\alpha,t}(x)=\bigg \{
\begin{array}{rl}
    &C_1^{1-2\alpha}p^{1-2\alpha}(x), \textrm{ for }x\in \Omega\\
    \\
    &C_3^{1-2\alpha}p^{1-2\alpha}(x), \textrm{ for }x\in \d\Omega\\
\end{array}
\end{equation}

Step 4: The normalized graph Laplacians can be obtained by
normalization through $d_{\alpha,t}(x)$. Then the limit of the
random walk normalized graph Laplacian is (we include the limit for
interior point $x$ for comparison)
\begin{equation}
\begin{array}{rl}
    \lim_{t\to 0} \frac{1}{t}L_{\alpha,t}^rf(x)=&-\frac{C_2}{2C_1}\Delta_s f(x), \textrm{ for }x\in \Omega\\
    \\
    \lim_{t\to 0} \frac{1}{\sqrt{t}}L_{\alpha,t}^rf(x)=&-\frac{C_4}{C_3}\d_\n f(x), \textrm{ for }x\in \d\Omega\\
\end{array}
\end{equation}
As for the limit of $L_{\alpha,t,n}^s$, it can be shown that
$L_{\alpha,t,n}^s f(x)= D_{\alpha,t,n}^{-1/2}L_{\alpha,t,n}^u F(x)$
where $F(x)=D_{\alpha,t,n}^{-1/2}f(x)$. Then the limit analysis
follows easily.

Step 5: Consider
\begin{equation}\label{equ:LaplacianSUM}
    \frac{1}{\sqrt{t}}L_{\alpha,t,n}^uf(x)=\frac{1}{n\sqrt{t}}\sum_{i=1}^n
    K_t(x,X_i)[d_{\alpha,t,n}(x)d_{\alpha,t,n}(X_i)]^{-\alpha}[f(x)-f(X_i)]
\end{equation}
Notice that in the sum, different terms are not independent, since
the degree $d_{\alpha,t,n}(x)$ and $d_{\alpha,t,n}(X_i)$ includes
sums of all the random variables. Therefore, we need to use the
McDiarmid's inequality in this step. The maximum change if we change
a random variable is bounded by
\begin{equation}
    \frac{1}{nt^{(d+1)/2}}\cdot \frac{1}{a^{2\alpha}} \cdot 2M
\end{equation}
The maximum change happens when we move a point $X_i$ from a high
density region with a minimum function value to a point $\hat{X}_i$
in a low density region with a maximum function value. Similar
analyses apply to normalized graph Laplacians. Then We conclude the
proof by applying the McDiarmid's inequality.
\end{proof}\\
Notice that the error rate essentially comes from the McDiarmid's
inequality. When $\alpha=0$, all terms in equation
(\ref{equ:LaplacianSUM}) are i.i.d., then we can use the Bernstein's
inequality to obtain a better rate for $L^u$. For $L^r$, an even
better rate can be obtained as shown by \cite{singer}. When
$\alpha\ne 0$, although strictly speaking the terms in equation
(\ref{equ:LaplacianSUM}) are not i.i.d., since $d_{\alpha,t,n}(x)$
is really an average of all the samples, it is almost a function of
$x$ alone, and $d_{\alpha,t,n}(X_i)$ a function of $X_i$ alone. Then
in this case, we believe it is possible to obtain a better error
rate.

Together with the existing analysis for interior points, we have the
following implication of Theorem (\ref{thm:limit:L})
\begin{equation}
\begin{array}{rcl}
\frac{1}{t} L_{\alpha,t,n}^rf(x) &\approx &
-\frac{C_2}{2C_1}\Delta_s
f(x), \textrm{ for }x\in \Omega\\
\\
\frac{1}{\sqrt{t}} L_{\alpha,t,n}^rf(x) &\approx &
-\frac{C_4}{C_3}\d_\n f(x), \textrm{ for }x\in \d\Omega
\end{array}
\end{equation}
Therefore, the graph Laplacian converges to a different limit on
$x\in \d\Omega$ from that on $x\in \Omega$. More importantly, these
two limits are of different orders, one is $O(t)$ while the other is
$O(\sqrt{t})$. However, in practice, when we apply the normalization
step, we do not know where the boundary is, and always apply a {\em
global} normalization $\frac{1}{t}$ for all $x\in \Om$ in order to
obtain the weighted Laplacian in the limit. For a small $t$

\begin{equation}
\frac{1}{t}L_{\alpha,t}^rf(x)=\bigg\{
\begin{array}{rl}
& -\frac{C_2}{2C_1}\Delta_s f(x)+O(t^{1/2}), \textrm{ for }x\in \Omega_0\\
\\
& -\frac{C_4}{C_3}\frac{1}{\sqrt{t}}\d_\n f(x)+O(1), \textrm{ for
}x\in \Omega_1
\end{array}
\end{equation}
Notice that the $O(1)$ error only happens on a ``shell'' $\Omega_1$
having volume $O(\sqrt{t})$. For $f(x)$ such that $\d_\n f(x)\ne 0$
on the boundary point $x$ with enough data points, we have that for
small values of $t$
\begin{equation}\label{equ:blowup}
 \frac{1}{t}L_{\alpha,t}^r f(x)=O\left(\frac{1}{\sqrt{t}}\right)
\end{equation}

\section{Numerical Examples }\label{sec:NumericalExample}
In this section, we explore the boundary behavior of the graph
Laplacian by studying numerical examples.
\begin{figure}[h]
\vskip 0.2in
\begin{center}
\subfigure[{$\frac{1}{t}L_{\alpha,t,n}^r f(x) \textrm{ over }
[1,2]$}]{
\includegraphics[width=0.3\columnwidth]{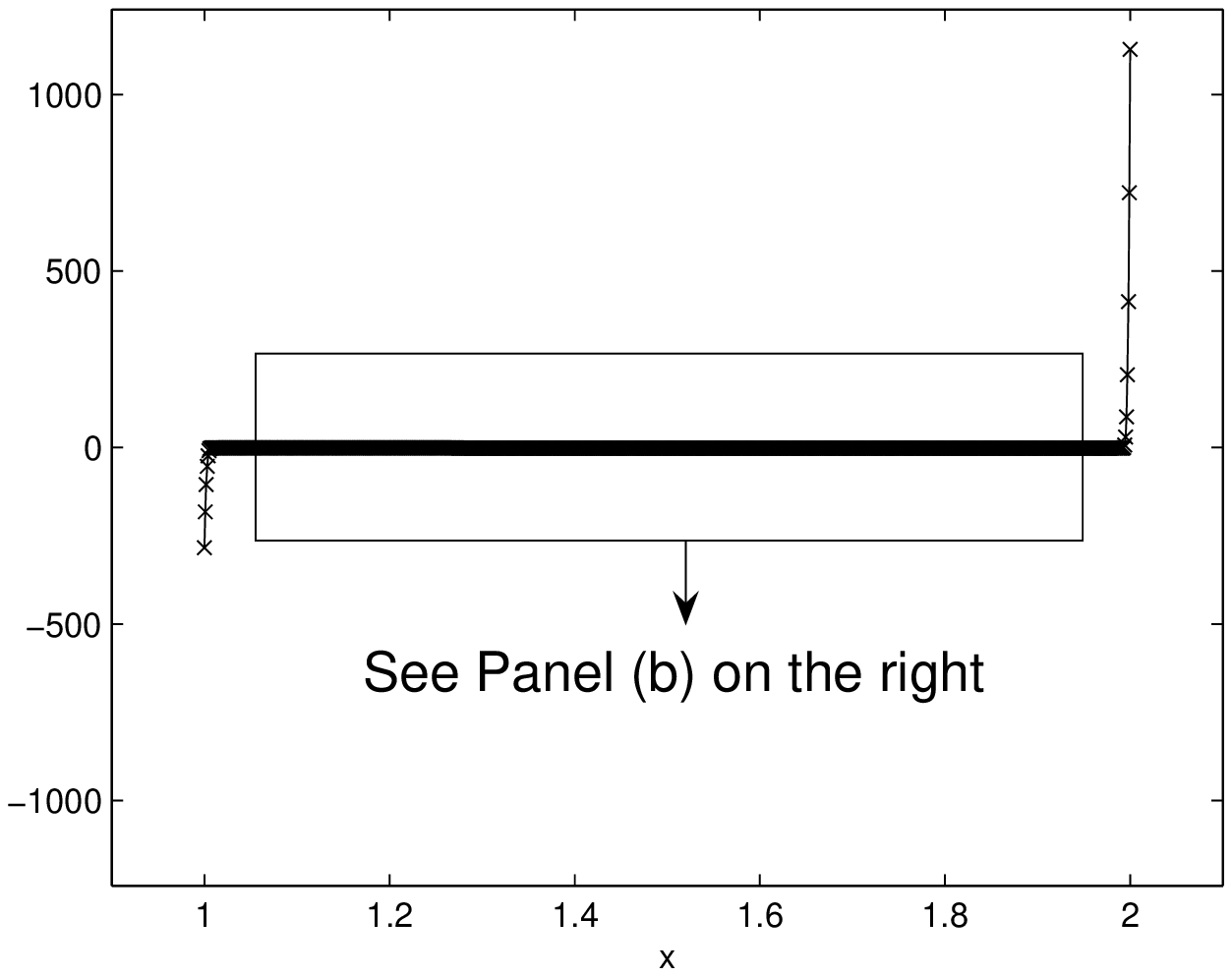}
}\subfigure[{$\frac{1}{t}L_{\alpha,t,n}^r f(x) \textrm{ over }
[1.1,1.9]$}]{
\includegraphics[width=0.3\columnwidth]{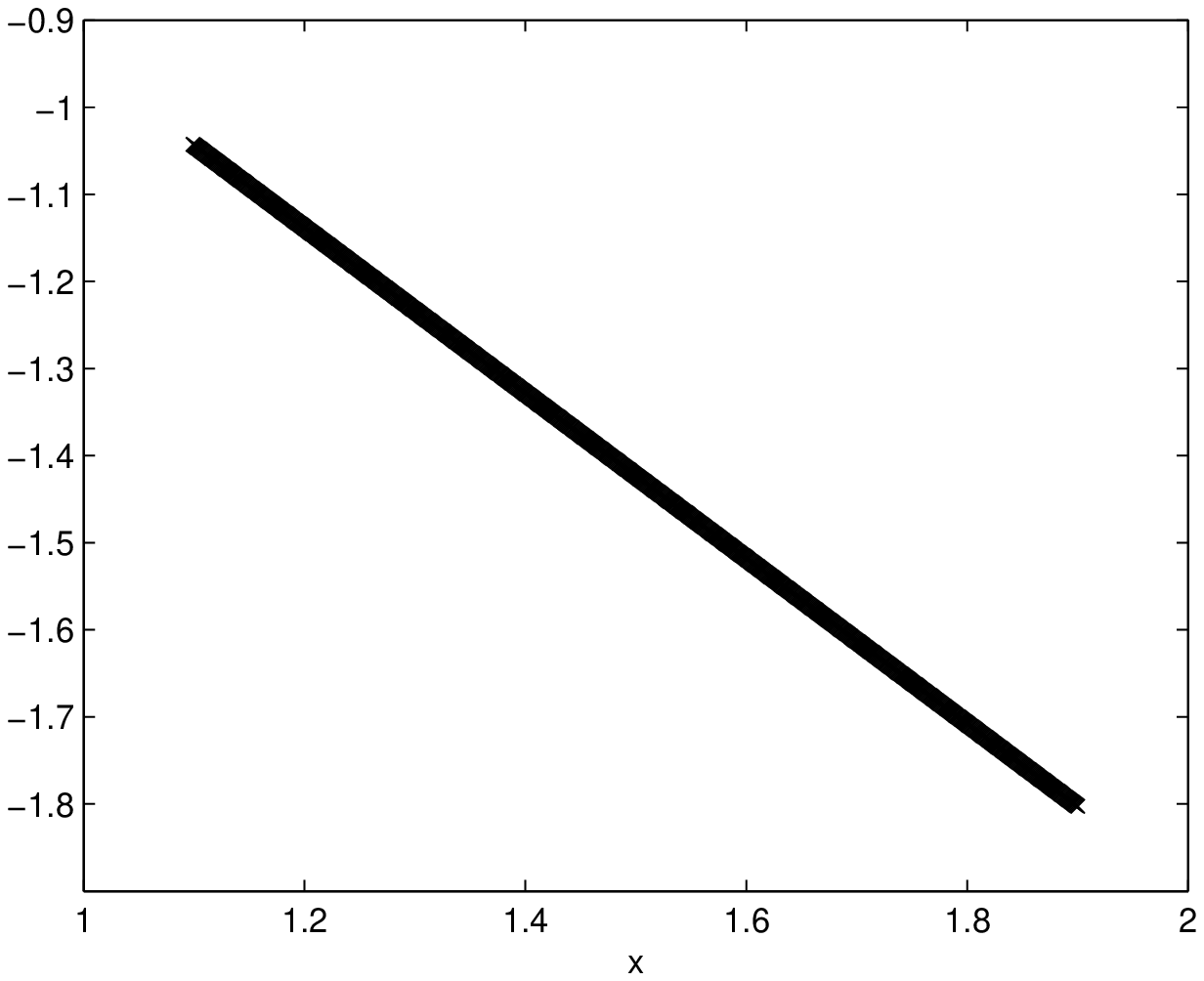}
} \subfigure[$\log(\frac{1}{t}|L_{\alpha,t,n}^r f(2)|)$ vs
$\log(t)$]{
\includegraphics[width=0.3\columnwidth]{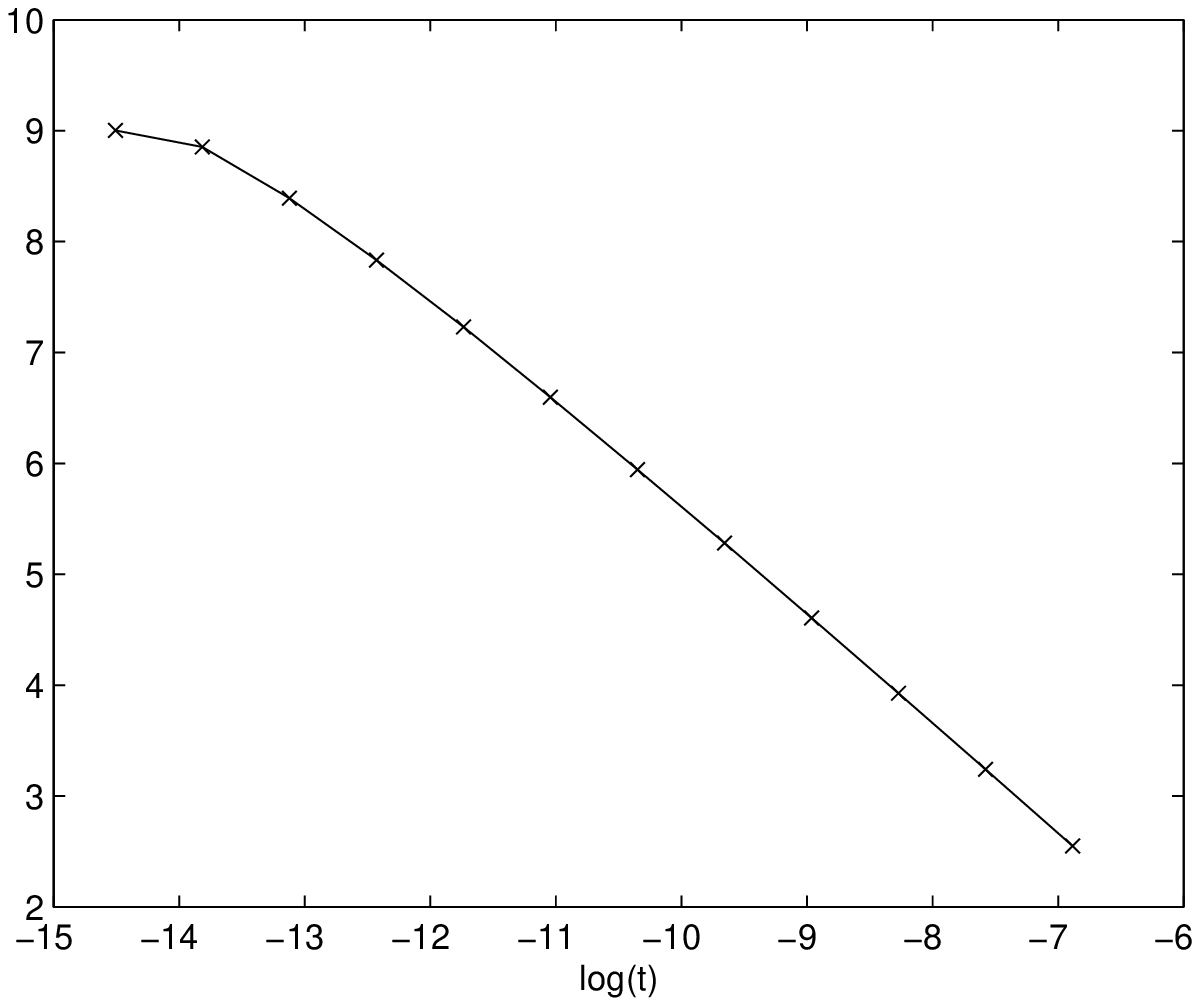}}
\caption{$\frac{1}{t}L_{\alpha,t,n}^r f(x)$ with $f(x)=x^3$ over
$[1,2]$.} \label{fig:blowup}
\end{center}
\end{figure}

\subsection{Graph Laplacian on the Boundary}
{\bf Example 1.}  We take  $\Om = [1,2]$, and  $f(x)=x^3$. The
values of  $\frac{1}{t}L_{\alpha,t,n}^r f(x)$ with $\alpha=0$ for
$1000$ points sampled from a uniform distribution (equal-spaced
points) and $t=10^{-5}$ are shown in Panel (a) in Figure
(\ref{fig:blowup}). As expected, we see that the values at the
boundary are much larger than those inside the domain and are
consistent with $-(x^3)' = -3x^2$ (the value at $2$ is roughly
$4$-times of the value at $1$) up to a scaling factor\footnote{The
positive value at $x=2$ is the result of the normal direction
pointing inward (left).}.

In Panel (b) we show the interior $[1.1, 1.9]$  of the interval
where the function is indeed the Laplacian $-(x^3)'' = -6x$ up to a
scaling factor. In Panel (c) we analyze the scaling of the graph
Laplacian on the boundary as a function of $t$ in the log-log
coordinates. We see that $\log \frac{1}{t}|L_{\alpha,t,n}^r f(2)|$
is close to a linear function of $\log(t)$ with slope approximately
$-\frac{1}{2}$ as you would expect from the scaling factor
$\frac{1}{\sqrt{t}}$.

{\bf Example 2}. Next we analyze  the boundary behavior for a simple low
dimensional manifold. Let $\Om$ be half a unit sphere ($z\ge
0$), which is a $2$-dimensional submanifold  in $\R^3$. The boundary
is a unit circle $\{(x,y,z):(x^2+y^2=1,z=0)\}$. We take
$f(x,y,z)=xz$, then the negative inward normal gradient on the
boundary is
\begin{equation}
-\d_\n f(x)=-(z,0,x)(0,0,1)^T=-x
\end{equation}
where $(z,0,x)$ is the gradient of $f(x,y,z)$, and $(0,0,1)$ is the
inward normal direction. This means the negative normal gradient of
$f$ along the inward normal direction on the boundary of a half
sphere is a linear function in $x$ with a negative slope.

We generate a uniform sample set of $2000$ points on a half sphere
and compute a vector $g=\frac{1}{\sqrt{t}}L_{\alpha,t,n}^r f(X)$
with $\alpha=0$, $t = 0.5$. We pick the set  $B=\{(x,y,z)\in X |
0\le z\le 0.05\}$ to be points near the boundary. The dependence
between $\frac{1}{\sqrt{t}}L_{\alpha,t,n}^r f(x,y,z)$ and $x$ for
$2000$ data points sampled from the uniform distribution on the half
sphere is plotted in Figure~(\ref{fig:bd:half:shpere}) and is
consistent with our expectation.

\begin{figure}[h]
\centering \epsfig{file=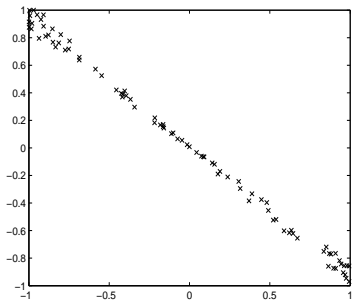, height=4.5cm, width=5.5cm}
\caption{$\frac{1}{\sqrt{t}}L_{\alpha,t,n}^r f(x)$ on the boundary
of a uniform half sphere, with $t=0.5$.}\label{fig:bd:half:shpere}
\end{figure}

To provide a more rigorous error analysis we compute the mean square
errors for several values of $t$ and sample size $n$. The results
are shown in Table~(\ref{table:error}). We see that the errors are
relatively small compared to the values of the gradient and
generally decrease with more data.

\begin{table}[h]
\centering \caption{Mean square errors between the analytical normal
gradient $-\d_\n f(x,y,z)=-x$ and
$\frac{1}{\sqrt{t}}L_{\alpha,t,n}^r f(x,y,z)$ on the boundary of a
half sphere with $f(x,y,z)=xz$.}\label{table:error}
\begin{tabular}{|c||c|c|c|c|c|c|c|} \hline
${n \textrm{\textbackslash} t}$& 64/64 & 32/64 & 16/64 & 8/64 & 4/64 & 2/64 &1/64 \tstrut \bstrut\\
\hline\hline
500 &  0.0090  &  0.0059  &  0.0071  &  0.0136  &  0.0287 & 0.0500 & 0.0725  \tstrut \bstrut\\
1000 &  0.0089  &  0.0048  &  0.0033  &  0.0061  &  0.0121  & 0.0294 & 0.0627 \tstrut \bstrut\\
2000  &  0.0113  &  0.0076  &  0.0073  &  0.0068  &  0.0159  &  0.0356 &  0.0585 \tstrut \bstrut\\
4000 &  0.0083  &  0.0044   & 0.0036  &  0.0044  &  0.0072  & 0.0189 &   0.0516 \tstrut \bstrut\\
\hline
\end{tabular}
\end{table}

\subsection{Comparison to Numerical PDE's}
\label{subsec:numericalPDE} From previous analysis we can see that,
for graph Laplacians, the ``missing'' edges going out of the
manifold boundary on one hand can be seen as being reflected back
into $\Om$, which is particularly intuitive in symmetric $k$NN
graphs, see e.g., \cite{maier2009}, on the other hand, it can be
seen that function values on edges going out of $\Om$ are constant
along the normal direction. The latter view is commonly used in
schemes of numerical PDE's in finite difference methods for the
Neumann boundary condition, see e.g., \cite{allaire}.

\begin{figure}[h]
\centering \epsfig{file=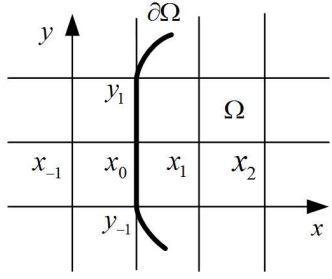, height=4.5cm, width=5.5cm}
\caption{Regular grid in $\R^2$.}\label{fig:grid}
\end{figure}

We use an example in $\R^2$ to show how the Neumann boundary
condition for a Laplace operator on a regular grid is implemented in
finite difference method, which we hope can shed light on the graph
Laplacian on random points. The Laplace operator in $\R^2$ is
$\Delta f=- \d_x^2 f- \d_y^2 f$, and the regular grid near the
boundary is shown in Figure~(\ref{fig:grid}). Since we can separate
the Laplacian into partial derivatives of different dimensions, the
discrete Laplace matrix $L$ on the regular grid with the Neumann
boundary condition near $x_0$ along $x$ direction can be shown to be
\begin{displaymath}
L=
\begin{array}{rl}
x_0\\
x_1\\
x_2
\end{array}
\Bigg [
\begin{array}{rrrrrr}
1  & -1 &  0 &  0 & \ 0 & \cdots \\
-1 &  2 & -1 &  0 & \ 0 & \cdots \\
0  & -1 &  2 & -1 & \ 0 & \cdots \\
\end{array}
\Bigg ]
\end{displaymath}
where we only connects points that are next to each other. Along $y$
direction the Laplace matrix elements near $x_0$ are $[\cdots \ -1\
\ 2 \ -1 \ \cdots]$. Let the distance between data points be $h$.
Consider point $x_0$ on the boundary, along $y$ direction, we have a
3-point stencil along $y$ axis, $y_{-1},y_0=x_0, y_1$, which is
enough to define $\d_y^2 f(x_0)$.
\begin{displaymath}
\lim_{h\to 0}\frac{1}{h^2}Lf(y_0)=-\lim_{h\to
0}\frac{f(y_{-1})-2f(y_0)+f(y_1)}{h^2} = -\d_y^2 f(y_0)
\end{displaymath}
This is also true for points that are in the interior $\Omega$, e.g.
$x_1$, $x_2$, etc. However, for $x_0$ on the boundary along $x$
direction (normal direction at $x_0$) we only have two points
\begin{displaymath}
\lim_{h\to 0}\frac{1}{h^2}Lf(x_0)=-\lim_{h\to
0}\frac{f(x_1)-f(x_0)}{h^2} \to -\lim_{h\to 0}\frac{\d_x f(x_0)}{h}
\end{displaymath}
This shows that $Lf(x)$ ``converge'' to $\frac{1}{h}\d_\n f(x)$ for
$x$ on the boundary while to $\Delta f(x)$ for $x$ inside of the
domain, with a different scaling behavior. This is almost the same
as what happens to the graph Laplacian on random samples. Notice in
numerical PDE's we also have that if $\d_x f(x_0)\ne 0$, as $h\to
0$, $\frac{\d_\n f(x_0)}{h}\to \infty$. In fact, if we construct the
graph Laplacian matrix $D-W$ by setting $w_{ij}=1$ if two points are
next to each other and $w_{ij}=0$ otherwise, on two dimensional grid
as shown in Figure~(\ref{fig:grid}), the graph Laplacian matrix is
the same as the Laplace matrix with the {\em Neumann boundary
condition} in numerical PDE's.

In order to let $Lf(x)$ converges to $\Delta f(x)$ for all $x$ on
domain $\Om$ with a single normalization term, we can add a
`fictitious' point $x_{-1}$ along the normal direction, and let
$f(x_{-1})=f(x_0)$. Then as $h\to 0$ we have
\begin{displaymath}
\frac{1}{h^2}Lf(x_0)=-\frac{f(x_1)-f(x_0)}{h^2}=-\frac{f(x_{-1})-2f(x_0)+f(x_1)}{h^2}
\to -\d_x^2 f(x_0)
\end{displaymath}
Together with $y$ direction, we have $\forall x\in \Om$, $\lim_{h\to
0} Lf(x)= -\d_x^2 f(x) -\d_y^2 f(x) =\Delta f(x)$. Condition
$f(x_{-1})=f(x_0)$ then becomes
\begin{displaymath}
\lim_{h\to 0} \frac{f(x_{-1})-f(x_0)}{h}=\d_x f(x_0)=0
\end{displaymath}
which is the Neumann boundary condition. This method is used to
implement the Neumann boundary condition in finite difference
methods for PDE's, see e.g., \cite[Chapter 2]{allaire}.

The graph Laplacian can be seen as an implementation of the Neumann
Laplacian on random points, which generalizes regular grids to
random graphs based on random samples. This also means by
construction, the graph Laplacian is a {\em Neumann Laplacian},
which is a built-in feature of the graph Laplacian.

\section{Discussions and Implications}\label{sec:implication}
\subsection{Neumann Boundary Condition} \label{sec:NeumannBC}
Our analysis of the boundary behavior suggests that the
eigenfunctions of both $L_\alpha^r$ and $L_\alpha^u$ as well as
solutions of certain regularization problems should satisfy the
Neumann boundary condition. However, this is not true for the
symmetric normalized Laplacian $L_\alpha^s$.

{\bf Unnormalized and Random Walk Normalized Graph Laplacians:}
These two versions of graph Laplacians only differ in the density
weight outside of the normal gradient, so we only need to find the
boundary condition for one of them. Let $\phi_i(x)$ and $\lambda_i$
be the $i^{th}$ right eigenfunctions of $L_\alpha^r$, then for any
positive integer $i$ and $x\in \d\Omega$, the following Neumann
boundary condition holds.
\begin{equation}
\d_\n \phi_i(x)=0
\end{equation}
This is also true for the eigenfunctions of $L_\alpha^u$, the limit
of the unnormalized graph Laplacian, which can be seen as follows.
All the eigenfunctions should satisfy
\begin{displaymath}
\lim_{t\to 0}\frac{1}{t}L_{\alpha,t}^r \phi_i(x)=L_\alpha^r
\phi_i(x)=\lambda_i \phi_i(x)
\end{displaymath}
On the boundary
\begin{displaymath}
\lim_{t\to 0}\frac{1}{t}L_{\alpha,t}^r \phi_i(x) = -\lim_{t\to 0}
\frac{1}{\sqrt{t}}\d_\n \phi_i(x)
\end{displaymath}
Since $\lambda_i \phi_i(x)<\infty$ for all $x\in \Om$, if
$\phi_i(x)$ does not satisfy the Neumann boundary condition, then
$\lim_{t\to 0}\frac{1}{t}L_{\alpha,t}^r \phi_i(x)\to \infty$ on the
boundary, therefore such $\phi_i(x)$ can not be the eigenfunctions
of $L_\alpha^r$. This implies that all the eigenfunctions of
$L_\alpha^r$ should satisfy the Neumann boundary condition, i.e.,
$\forall i, \d_\n \phi_i(x)=0$ for $x\in \d\Omega$. Similarly, this
is true for the unnormalized graph Laplacian with a bounded density.

\begin{figure}[h]
\vskip 0.2in
\begin{center}
\subfigure[$\phi_2(x)$ for uniform over unit interval.]{
\includegraphics[width=0.33\columnwidth]{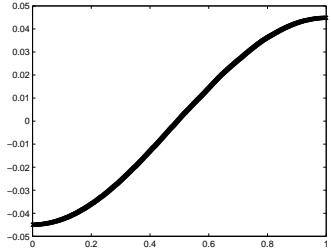}
}\hspace{45pt} 
\subfigure[$\phi_2(x)$ for mixture of two Gaussians.]{
\includegraphics[width=0.33\columnwidth]{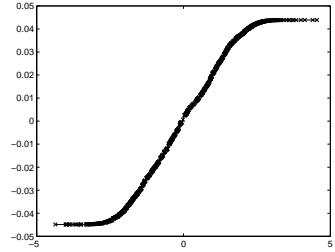}
} \subfigure[$\phi_3(x)$ for uniform over unit interval.]{
\includegraphics[width=0.33\columnwidth]{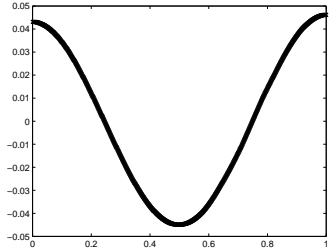}
}\hspace{45pt} 
\subfigure[$\phi_3(x)$ for mixture of two Gaussians.]{
\includegraphics[width=0.33\columnwidth]{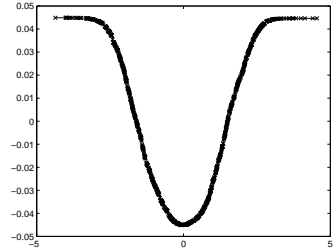}
} \caption{The eigenfunctions of the graph Laplacian for a uniform
density and a mixture of two Gaussians centered at $\pm1.5$ with
unit variance in $\R^1$.} \label{fig:eigenfunctions:R1}
\end{center}
\end{figure}

We numerically compute the second and third eigenfunctions of
$L_{\alpha,t,n}^r$ with $\alpha=1/2$ in
Figure~(\ref{fig:eigenfunctions:R1})\footnote{The Neumann boundary
condition also holds for other $\alpha$.}. The left panel shows the
eigenfunctions over interval $[0,1]$ with a uniform density, while
the right panel is for a mixture of two Gaussians. As the numerical
results suggest, the second and third eigenfunctions of the graph
Laplacian satisfy the Neumann boundary condition, and this is also
true for other eigenfunctions. In fact for a uniform over $[0,1]$,
the Neumann eigenfunctions for the Laplacian are $\cos(k\pi x)$
where $k=0,1,2,\cdots$. The second and third eigenfunctions
correspond to $\cos(\pi x)$ and $\cos(2\pi x)$, up to a change of
sign, which is consistent with the numerical results in the left
panel of Figure~(\ref{fig:eigenfunctions:R1}).

{\bf Symmetric Normalized Graph Laplacian:} For the symmetric
normalized graph Laplacian $L_{\alpha,t,n}^s$, the Neumann boundary
condition does not hold for its eigenfunctions in the limit. This
can be shown by a one to one correspondence between the
eigenfunctions of $L_{\alpha,t,n}^s$ and $L_{\alpha,t,n}^s$. Let the
eigenvectors for $L_{\alpha,t,n}^s$ be $\psi_i$, and the right
eigenvectors for $L_{\alpha,t,n}^r$ be $\phi_i$, then
\begin{displaymath}
\psi_i(X_i)=d^{1/2}(X_i)\phi_i(X_i)
\end{displaymath}
This is true for any sample size and any parameter $t$. Therefore,
in the limit
\begin{displaymath}
\d_\n \psi_i(x)=\d_\n [d^{1/2}(x)\phi_i(x)]=\phi_i(x) \d_\n
d^{1/2}(x)
\end{displaymath}
Near the boundary along the normal direction, degree function $d(x)$
decreases as a result of the asymmetric interval for $\int_{\Om}
K_t(x,y)p(y) dy$, so $\psi_i(x)$ tends to be ``bent'' towards zero
near the boundary. Since how the graph is constructed will determine
what the degree function will be, the boundary behavior also depends
on what graph is used. We test two graphs, $\epsilon$NN graph and
symmetric $k$NN graph, which are studied by \cite{maier2009} for
clustering. In the left panel of Figure~(\ref{fig:eigenfunctionLs}),
$d(x)$ is scaled and shifted to fit the plot and an $\epsilon$NN
graph is used. For $x$ near the boundary, the degree function
decreases as a result of having less points in the fixed radius
neighborhood. Therefore, $\psi_2(x)$ is ``bent'' towards zero.

Notice that for the symmetric $k$NN graph, in the right panel of
Figure~(\ref{fig:eigenfunctionLs}), the eigenfunctions can have
``bumps'' near the boundary. This is the result that for symmetric
$k$NN graphs, the edges going out of the boundaries are
``reflected'' back. For example consider $k=10$ in a symmetric $k$NN
graph, with the distance between points set as $0.01$, for point
$x=0$, its nearest neighbors are $x=0.01,0.02,\cdots,0.1$. However,
for point $x=0.1$, $x=0$ is not in its $k$ nearest neighbors. This
means the constructed graph will not be symmetric. If we add
$e_{ji}$ for every asymmetric edge $e_{ij}$, then although the graph
is symmetric now, $d(x)$ for point $x=0.1$ will be much larger than
other points. As $t$ decreases, this ``bumps'' will shift to the
boundary.
\begin{figure}[h]
\vskip 0.2in
\begin{center}
\subfigure[On an $\epsilon$NN graph.]{
\includegraphics[width=0.4\columnwidth]{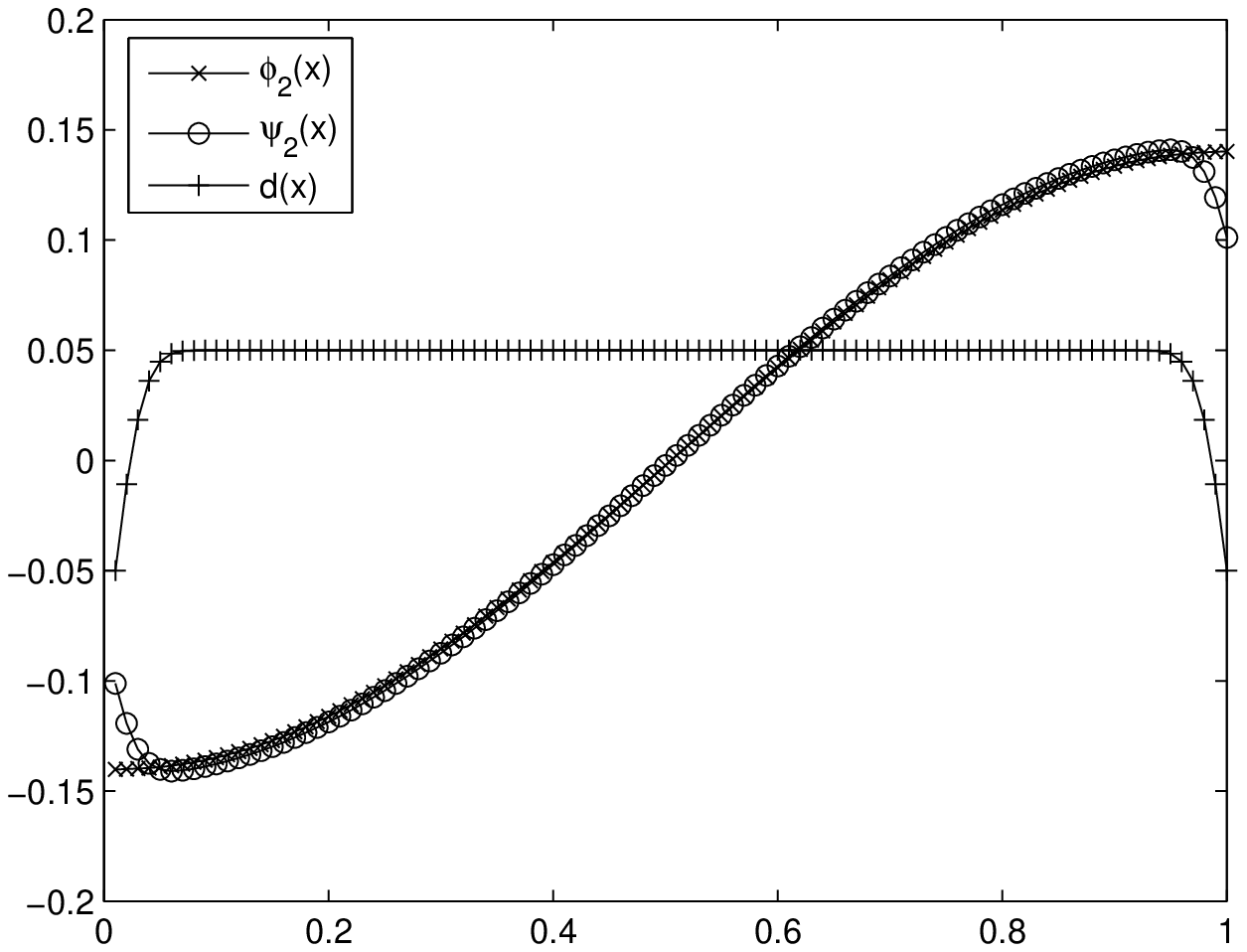}
}\hspace{30pt} \subfigure[On a symmetric $k$NN graph.]{
\includegraphics[width=0.4\columnwidth]{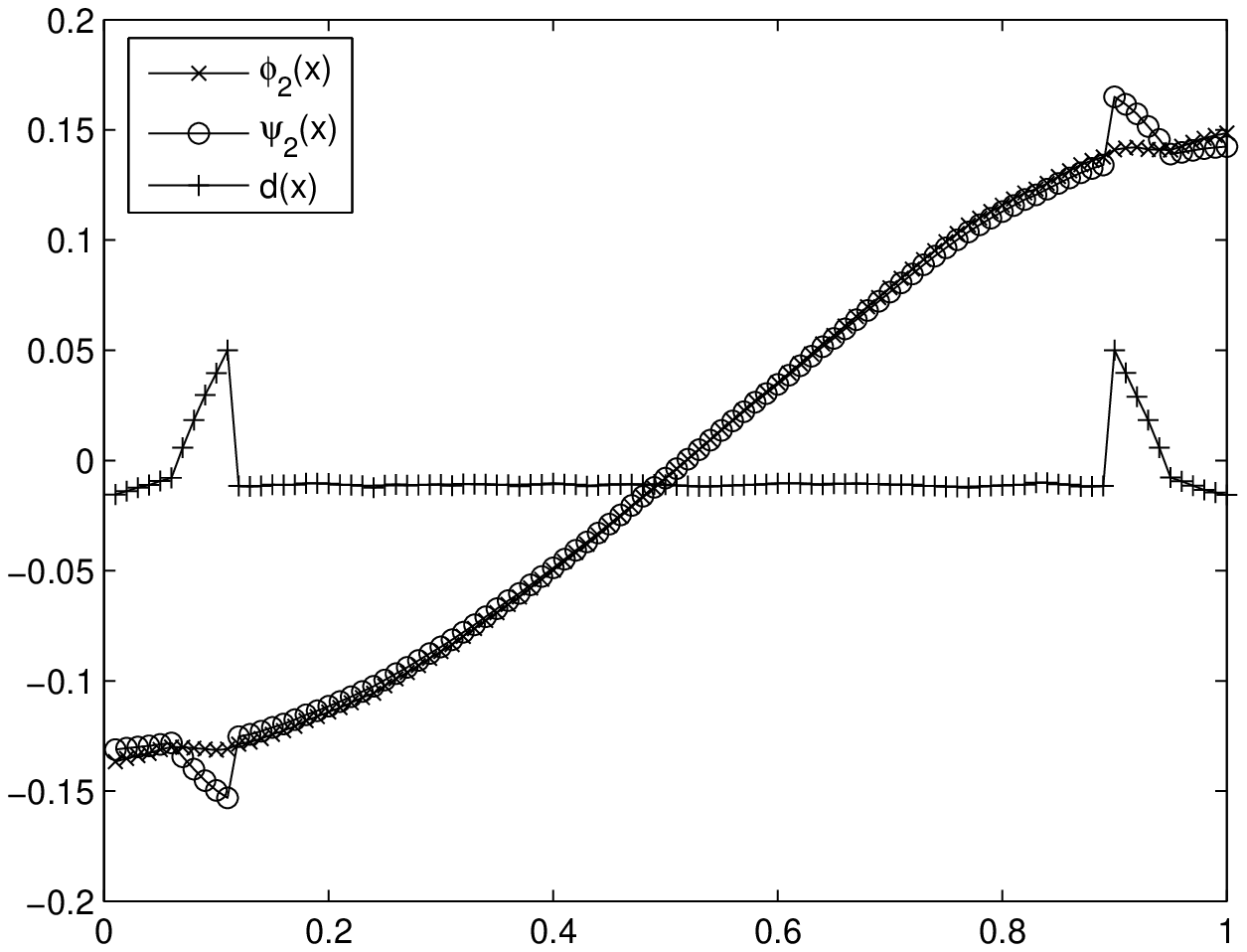}}
\caption{$\psi_2(x)$, $\phi_2(x)$ and $d(x)$ for uniform over
$[0,1]$.} \label{fig:eigenfunctionLs}
\end{center}
\end{figure}

\subsection{Limit of Graph Laplacian Regularizer}
The following graph Laplacian regularizer is a popular penalty term
in many semi-supervised learning algorithms when $\alpha=0$.
\begin{equation}\label{equ:quadraticform}
f^TL_{\alpha,t,n}^uf=\frac{1}{2}\sum_{i,j=1}^n
\frac{K_t(X_i,X_j)}{[d_{\alpha,t}(X_i)d_{\alpha,t}(X_j)]^\alpha}(f(X_i)-f(X_j))^2
\end{equation}
This limit is studied in \cite[Chapter 2]{hein} without considering
the boundary, and is also studied on $\R^N$ by \cite{bousquet},
which has no boundary and is not a low dimensional manifold either.
From Theorem (\ref{thm:limit:L}), we see that the graph Laplacian
has a limit of different scaling behavior on the boundary points
compared to the interior points. This leads to the question of the
limit of the graph Laplacian regularizer $f^TL_{\alpha,t,n}^uf$ when
it is defined on a compact submanifold with a smooth boundary. Based
on the quadratic form, we use a similar method as the proof of
Theorem (\ref{thm:limit:L}) to obtain the next theorem.

\begin{theorem}\label{thm:square}
For a fixed function $f(x)\in C^1(\Om)$, and let $p(x)\in
C^{\infty}(\Om)$, $0<a\le p(x)\le b <\infty$,  and the intrinsic
dimension of $\Om$ be $d$, then as $n\to \infty$, $t\to 0$ and
$n^2t^{d+2}\to \infty$,
\begin{equation}
\lim_{n\to \infty} \frac{1}{n^2 t}f^TL_{\alpha,t,n}^uf= C\int_{\Om}
\|\nabla f(x)\|^2 [p(x)]^{2-2\alpha} dx, \qquad \textrm{in
probability}
\end{equation}
where $C=\frac{1}{4}\pi^{d(1/2-\alpha)}$.
\end{theorem}

\begin{proof}
Following the proof of Theorem (\ref{thm:limit:L}), let $\Omega_1$
be a thin lay of ``shell'' of width $O(\sqrt{t})$, and
$\Omega_0=\Om/\Omega_1$. For a fixed $t$, consider the quadratic
form (\ref{equ:quadraticform}), the limit as $n\to \infty$ is
\begin{equation}
\frac{1}{2}\iint_{\Om}
\frac{K_t(x,y)}{d_{t}^\alpha(x)d_{t}^\alpha(y)}(f(x)-f(y))^2
p(y)p(x)dydx
\end{equation}
Then by the approximation on manifolds (\ref{equ:basicApprox}) and
(\ref{equ:detailedApprox}), for a fixed $x\in \Om$,
\begin{equation}
\begin{array}{rl}
& \frac{1}{2t^{d/2}}p(x)d_{t}^{-\alpha}(x)\int_{\Om}
K(\frac{\|x-y\|^2}{t})d_{t}^{-\alpha}(y)(f(x)-f(y))^2 p(y)dy\\
\\
= & \frac{1}{2t^{d/2}}p(x)d_{t}^{-\alpha}(x)\int_{\T(x)}
K(\frac{\|u\|^2}{t})d_{t}^{-\alpha}(x)(u^T \nabla f(x))^2 p(x)du+O(t^{3/2})\\
\\
= & \frac{1}{2t^{d/2}}p^2(x)d_{t}^{-2\alpha}(x) \|\nabla
f(x)\|^2\int_{\T(x)}
K(\frac{\|u\|^2}{t}) u_1^2 du+O(t^{3/2})\\
\\
= & \frac{t}{2}p^2(x)d_{t}^{-2\alpha}(x) \|\nabla
f(x)\|^2\int_{\T(x)} K(\|u\|^2) u_1^2 du
+O(t^{3/2})\\
\end{array}
\end{equation}
Notice that in this case, the highest order is controlled by
$(f(x)-f(y))^2$, which is of order $O(t)$, and $\T(x)$ is the
tangent space at $x$. Then for any point $x\in \Om$, the limit is
\begin{equation}
\frac{C_2(x)}{2C_1^{2\alpha}(x)}\|\nabla f(x)\|^2 [p(x)]^{2-2\alpha}
\end{equation}
where
\begin{equation}
\begin{array}{rl}
C_2(x)=&\int_{-\infty}^{+\infty}\cdots \int_{-\infty}^{+\infty} \int_{-z}^{+\infty} e^{-\|u\|^2}u_1^2 du\\
\\
C_1(x)=&\int_{-\infty}^{+\infty}\cdots \int_{-\infty}^{+\infty} \int_{-z}^{+\infty} e^{-\|u\|^2} du\\
\end{array}
\end{equation}
with $z$ as the distance between $x$ and the nearest boundary point
along the normal direction.

On $\Omega_0$, for a small $t$, we can replace $z$ with $\infty$,
then the integral on $\Omega_0$ is
\begin{equation}
\frac{C_2}{2C_1^{2\alpha}}\int_{\Omega_0} \|\nabla f(x)\|^2
[p(x)]^{2-2\alpha} dx
\end{equation}
where the coefficient becomes a constant independent of $x$.

On the shell $\Omega_1$, as $t\to 0$, the shell shrinks into a set
with measure zero. As long as the function inside the integral is
bounded, the integral on $\Omega_1$ will also be zero. For $x\in
\Omega_1$, we have $0\le z< +\infty$ and
\begin{equation}
\begin{array}{rcl}
\frac{1}{8}\pi^{d/2}\le & C_2(x)=\frac{1}{4}\pi^{d/2}(1-\frac{2z}{\sqrt{\pi}e^{z^2}}+\textrm{erf}(z))& \le \pi^{d/2}\\
\\
\frac{1}{2}\pi^{d/2} \le & C_1(x)=\frac{1}{2}\pi^{d/2}(1+\textrm{erf}(z)) &\le \pi^{d/2} \\
\end{array}
\end{equation}
For $f\in C^1(\Om)$ and any $x\in \Om$, in any direction, $|\frac{\d
f(x)}{\d x_i}|<\infty$. The density $p(x)$ is also bounded,
therefore, the integral on $\Omega_1$ is zero as $t\to 0$. Overall,
as $t\to 0$, $C_2(x)\to C_2$, $C_1(x)\to C_1$, and $\Omega_0$
becomes $\Om$, therefore,
\begin{equation}
\lim_{t\to 0}\frac{1}{2t}\iint_{\Om}
\frac{K_t(x,y)}{d_{t}^\alpha(x)d_{t}^\alpha(y)}(f(x)-f(y))^2
p(y)p(x)dydx =\frac{C_2}{2C_1^{2\alpha}}\int_{\Om} \|\nabla f(x)\|^2
[p(x)]^{2-2\alpha} dx
\end{equation}
Next consider
\begin{equation}
\frac{1}{n^2 t}f^TL_{\alpha,t,n}^uf=\frac{1}{2n^2
t}\sum_{i,j}K_t(X_i,X_j)[d_{\alpha,t,n}(X_i)d_{\alpha,t,n}(X_j)]^{-\alpha}[f(X_i)-f(X_j)]^2
\end{equation}
Since all the $n^2$ terms in the sum are not i.i.d., we use the
McDiarmid's inequality. The maximum change of replacing one random
variable is
\begin{equation}
\frac{1}{2n^2 t^{d/2+1}} \cdot a^{-2\alpha} \cdot (2M)^2
\end{equation}
therefore,
\begin{equation}
P(|\frac{1}{n^2 t}f^TL_{\alpha,t,n}^uf- C\int_{\Om} \|\nabla
f(x)\|^2 [p(x)]^{2-2\alpha} dx|>\epsilon)\le 2
e^{-\frac{n^2t^{d+2}\epsilon^2 a^{4\alpha}}{2M^4}}
\end{equation}
We conclude the proof by plugging in $C_1=\pi^{d/2}$ and
$C_2=\frac{1}{2}\pi^{d/2}$.
\end{proof}\\
One important implication of this theorem is that, in order to use a
gradient penalty term w.r.t. to different density weights in the
form of $p^s(x)$ in the limit, we can use $f^TL_{\alpha,t,n}^uf$
with different $\alpha$ values. For instance, to obtain
$\int_{\Om}\|\nabla f(x)\|^2p(x)dx$ instead of having $p^2(x)$ as
the commonly used penalty term $f^TL^uf$, we can set $\alpha=1/2$.
This penalty then fits the fact that sample $X_i$ are drawn from
density $p(x)$.

We tested Theorem~(\ref{thm:square}) numerically on several
functions with a uniform density over $1$-dimensional interval
$[0,1]$. $1001$ equal-space points are generated, the value of $t$
is between $1$ and $10^{-7}$, and we compute the numerical graph
Laplacian approximation $\frac{1}{n^2t}f^TL_{\alpha,t,n}^u f$, and
the analytical value of $\int_0^1 |f^\prime(x)|^2 dx$. The ratios of
the two are reported in Table (\ref{table:square}) and the
coefficient plots as a function of $\log(t)$ are shown in
Figure~(\ref{fig:square}).

\begin{table}[t]
\centering \caption{Coefficient test for different functions. The
largest value for different $t$ of each function is
reported.}\label{table:square}
\begin{tabular}{|c||c|c|c|c|c|c|c|c|c|} \hline
$\alpha$\textbackslash $f(x)$& $\sqrt{x+1}$ & $x$ & $x^2+10x$ & $x^3$ & $e^{x}$ & $\sin(x)$ & $\cos(x)$ & $\cos(10x)$ & $\frac{1}{4}\pi^{(1/2-\alpha)}$ \tstrut \bstrut\\
\hline\hline
0   &  0.4424  &  0.4424  &  0.4424  &  0.4420  &  0.4423  & 0.4424  &  0.4423 & 0.4426 & 0.4431 \tstrut \bstrut\\
1/2 &  0.2497  &  0.2497  &  0.2497  &  0.2497  &  0.2497  & 0.2497  &  0.2497 & 0.2497 & 0.2500 \tstrut \bstrut\\
1   &  0.1411  &  0.1411  &  0.1411  &  0.1412  &  0.1411  & 0.1412  &  0.1410 & 0.1411 & 0.1410 \tstrut \bstrut\\
-1  &  1.3845  &  1.3846  &  1.3846  &  1.3819  &  1.3840  & 1.3827  &  1.3865 & 1.3859 & 1.3921 \tstrut \bstrut\\
\hline
\end{tabular}
\end{table}

In Figure~(\ref{fig:square}), eight different functions from
Table~(\ref{table:square}) are tested for the coefficient $C$ using
different bandwidth $t$. Each plot corresponds to one fixed $\alpha$
value. As suggested by the figure, in Table~(\ref{table:square}),
the maximum values from different $t$ are reported. From both the
figure and the table we can see that, the numerical results are
close to the theoretical coefficient $\frac{1}{4}\pi^{1/2-\alpha}$.
The figure also suggests a numerically stable patterns as $t$
decreases, until it is too small and loses numerical precision.
\begin{figure}[h]
\vskip 0.2in
\begin{center}
\subfigure[$\alpha=0$.]{
\includegraphics[width=0.40\columnwidth]{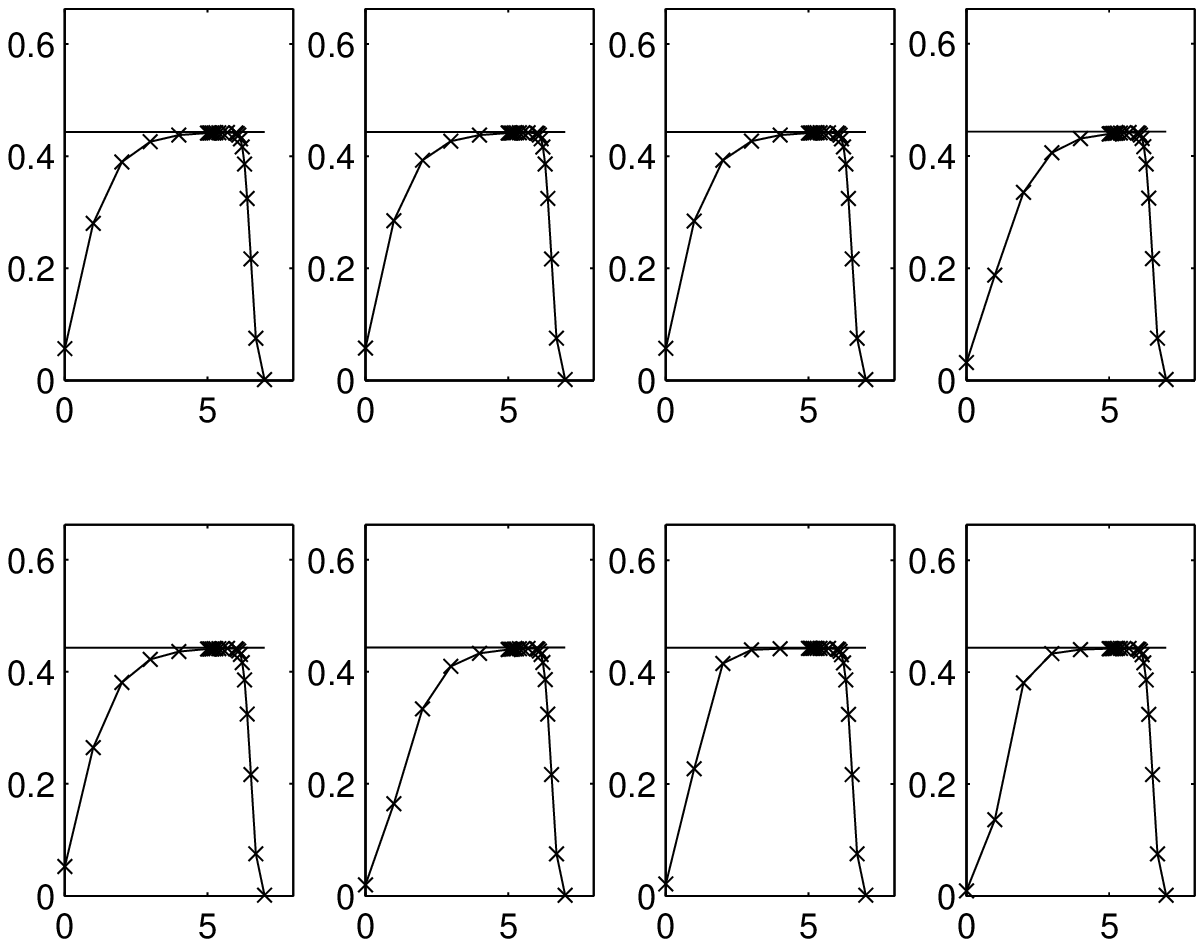}
}\subfigure[$\alpha=1/2$.]{
\includegraphics[width=0.40\columnwidth]{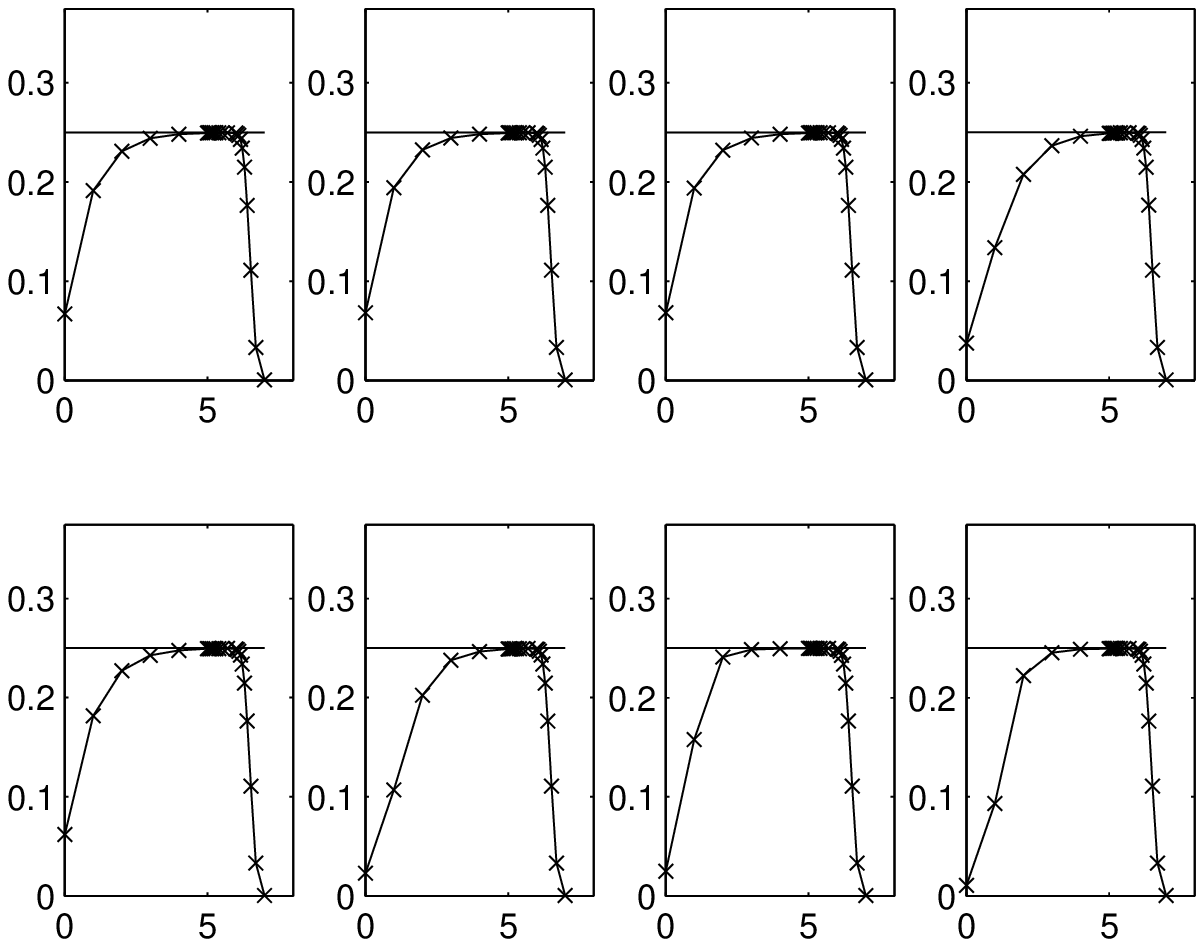}}
\subfigure[$\alpha=1$.]{
\includegraphics[width=0.40\columnwidth]{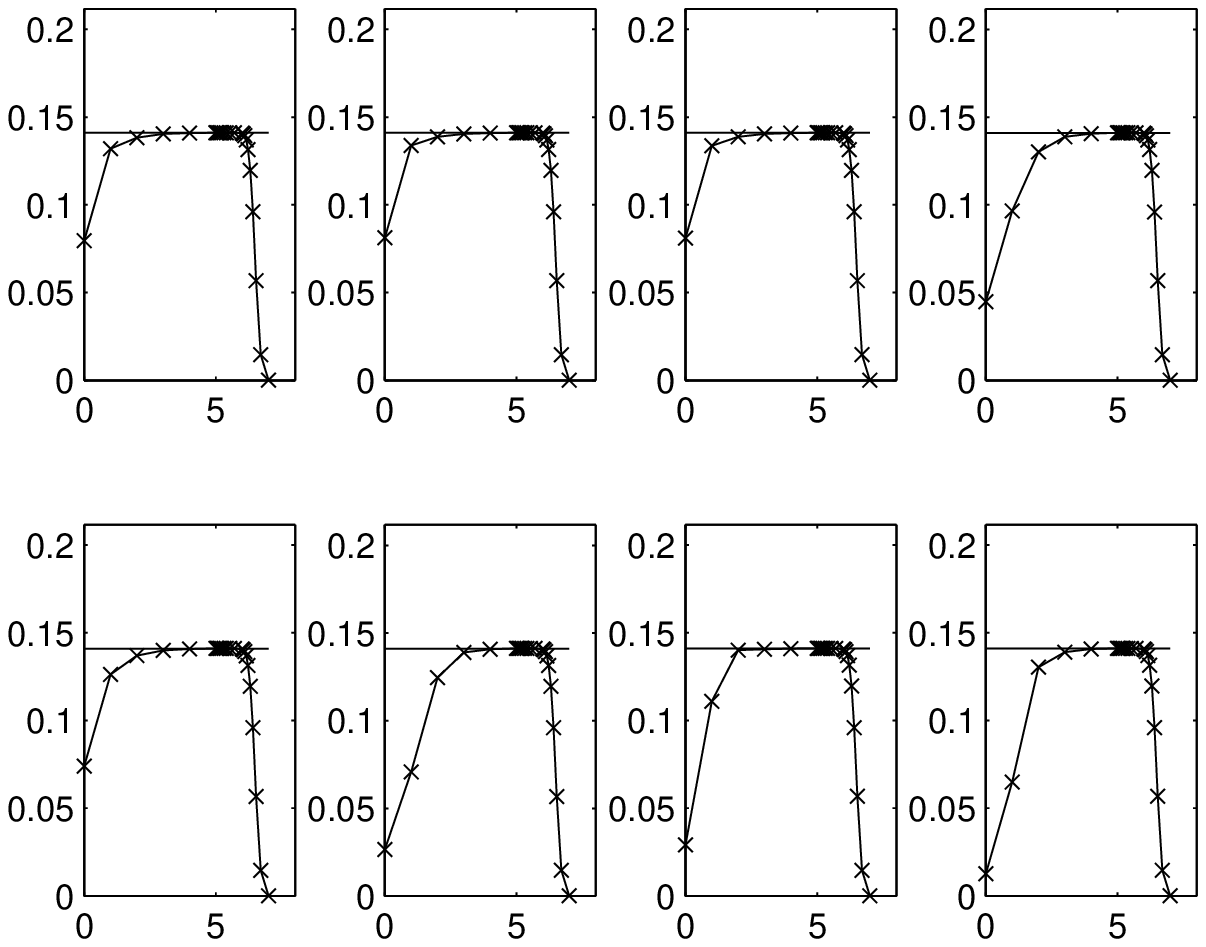}
}\subfigure[$\alpha=-1$.]{
\includegraphics[width=0.40\columnwidth]{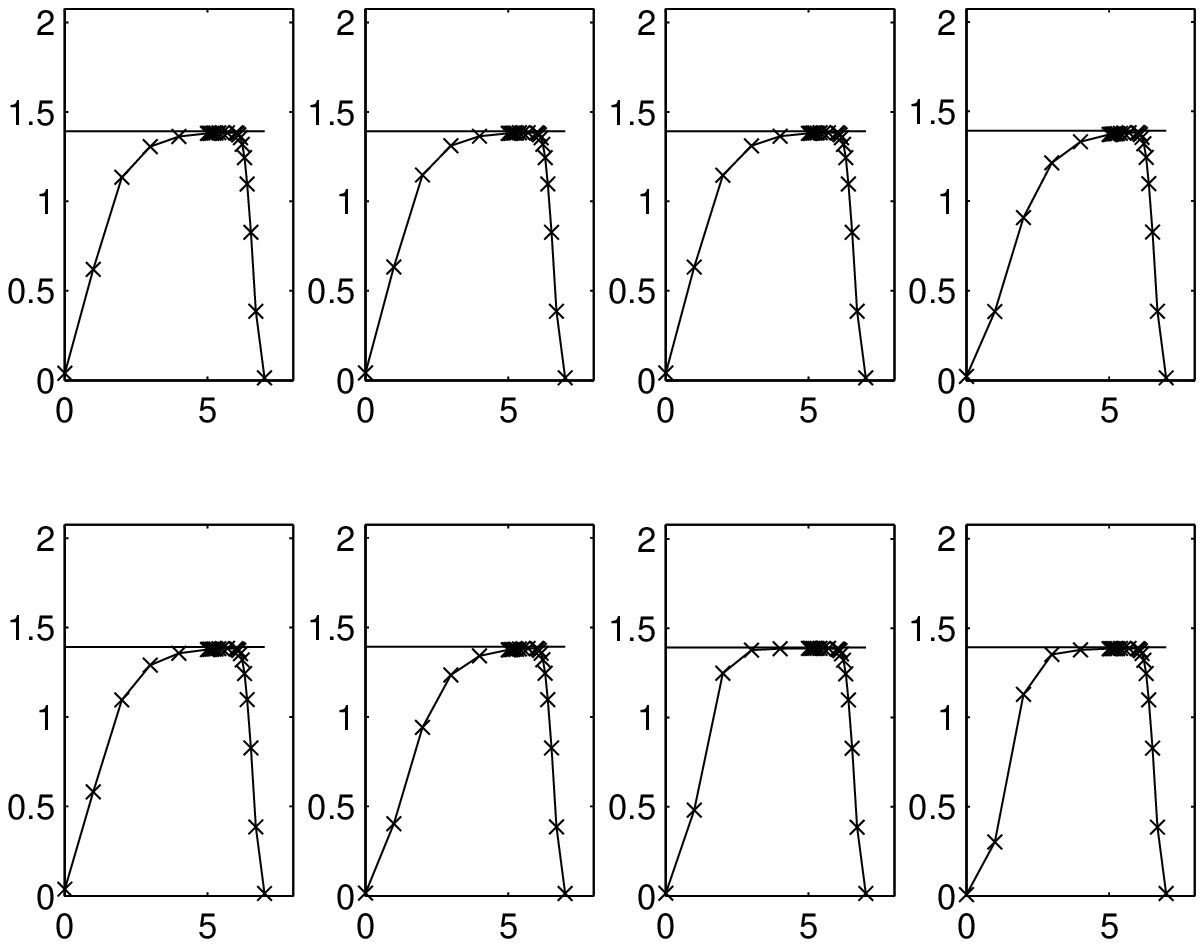}}
\caption{Coefficient $C$ as a function of $t$ for different
functions. The solid line is the theoretical result, and the $x$
axis corresponds to $-\log(t)$.}\label{fig:square}
\end{center}
\end{figure}

Notice that on finite samples, it is also possible to treat
$f^TL_{\alpha,t,n}^uf$ as an inner product as $\langle f,
L_{\alpha,t,n}^u f \rangle$. However, in the limit, as we see in
this paper, $L_{\alpha,t,n}^u f$ can degenerate to an unbounded
value on the boundary. Although this behavior only happens on a
small part having volume $O(t^{1/2})$ and the degenerating behavior
scales as $t^{1/2}$, they cannot cancel out each other since we can
not bring the limit $t\to 0$ across the integral when the sequence
of functions inside of the integral have an unbounded limit, i.e.,
it violates the condition of Dominated Convergence Theorem. When $f$
satisfy the Neumann boundary condition or $\Om$ has no boundary,
then it is safe to compute the limit as an inner product, which is
essentially the Green's first identity.

\subsection{Reproducing Kernels and Boundary Effects} In the short discussion below we
would like to illustrate the importance  of boundary effects in a
simple 1-dimensional setting. Consider the regularizer $f^TL^uf$,
whose limit for a fixed $f$ in $\R^N$ has the following
expression~\cite{bousquet}.
\begin{displaymath}
J(f)=\int_{\R^N} \|\nabla f(x)\|^2 p^2(x)dx
\end{displaymath}
In $\R^1$, the subspace orthogonal to the null space of $J(f)$ is a
reproducing  kernel  Hilbert space (RKHS) \cite{nadler2009}.

Consider its reproducing kernel function $K(x,y)$. Over the unit
interval $[0,1]$ with uniform probability density the kernel function can be found explicitly
by eigenfunction expansion of the Green's function
of the weighted Laplacian (Green's function in this case is the same
as kernel $K(x,y)$). Using the Neumann boundary condition, we find
that the  reproducing kernel (in the subspace orthogonal to the
null space) has the following expression:
\begin{equation}\label{equ:eigenexpansion}
K(x,y)=\sum_{k=1}^\infty \frac{1}{(k\pi)^2}\cos(k\pi x)\cos(k\pi y)
\end{equation}
On the other hand, in~\cite{nadler2009} the kernel without boundary conditions is shown
to be
\begin{equation}
K^\prime(x,y)=\frac{1}{4}-\frac{1}{2}|x-y|
\end{equation}
which is a different function.

In order to test our analysis, we notice that in finite sample case
the discrete Green's function of $L_{\alpha,t,n}^u$ is the same as
the reproducing kernel functions for space
$\{f:f^TL_{\alpha,t,n}^uf<\infty\}$, which is the pseudoinverse of
the matrix $L_{\alpha,t,n}^u$ \cite[Chapter 6]{berlinet}. In
Figure~(\ref{fig:kernel:R1}), we compute and plot the kernels
numerically to verify the above analysis. On one hand, we  use the
eigenfunction expansion (\ref{equ:eigenexpansion}) to find the
kernel. On the other hand, we  build the graph Laplacian matrix and
compute its pseudoinverse to obtain the approximate kernel function.
As shown in Figure~(\ref{fig:kernel:R1}), we can see that the kernel
obtained by eigenfunction expansion analytically is very close to
the one obtained from the pseudoinverse of the matrix
$L_{\alpha,t,n}^u$ (up to a constant scaling factor). This kernel is
quite different from $K^\prime$.  The difference  is due to the
global effect of the boundary behavior of the graph Laplacian and
provides additional evidence for the Neumann boundary condition in
eigenfunctions.

\begin{figure}[ht]
\vskip 0.2in
\begin{center}
\subfigure[Kernel $K$ by eigenfunction expansion]{
\includegraphics[width=0.3\columnwidth]{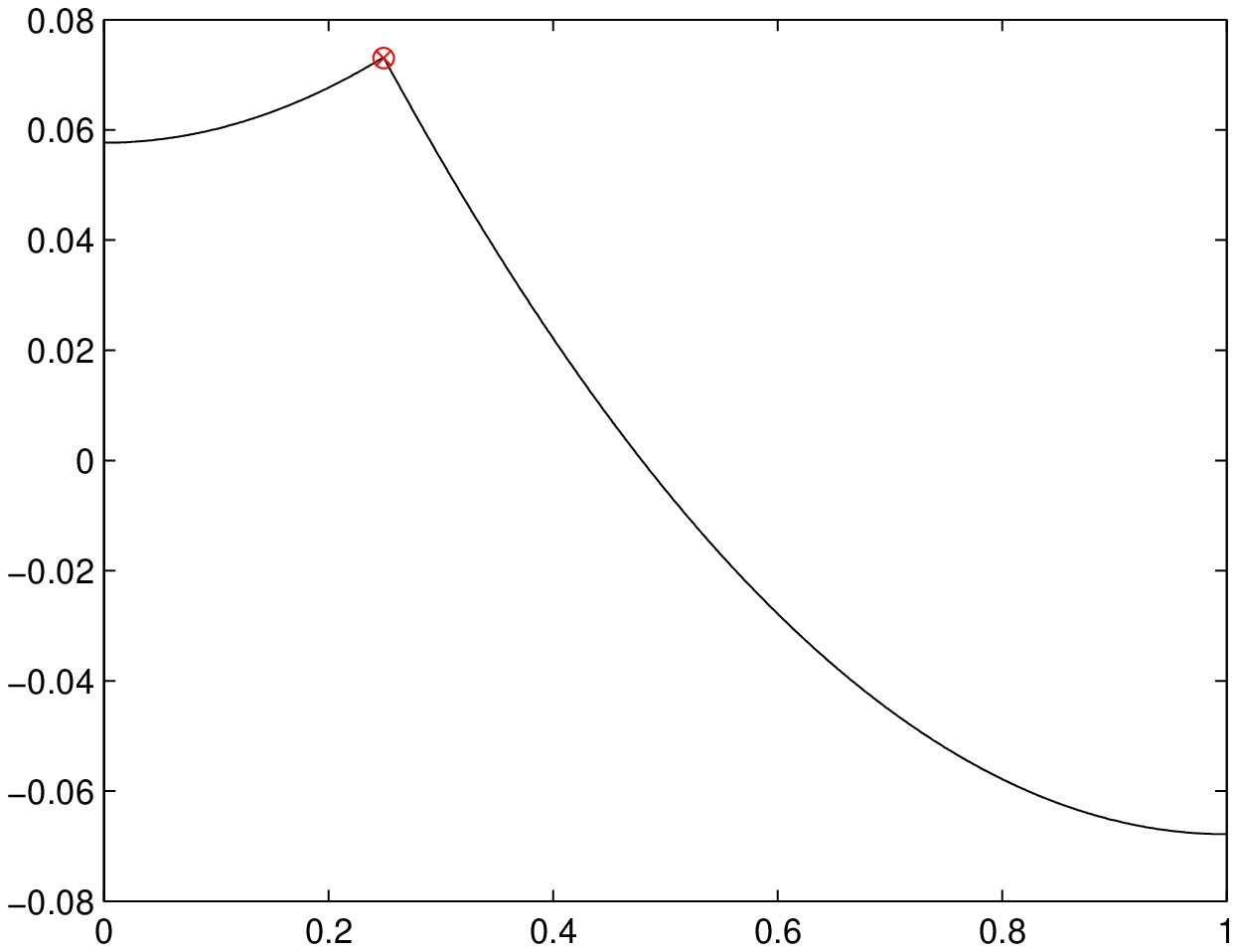}}
\subfigure[Kernel $K$ by pseudoinverse of the graph Laplacian]{
\includegraphics[width=0.3\columnwidth]{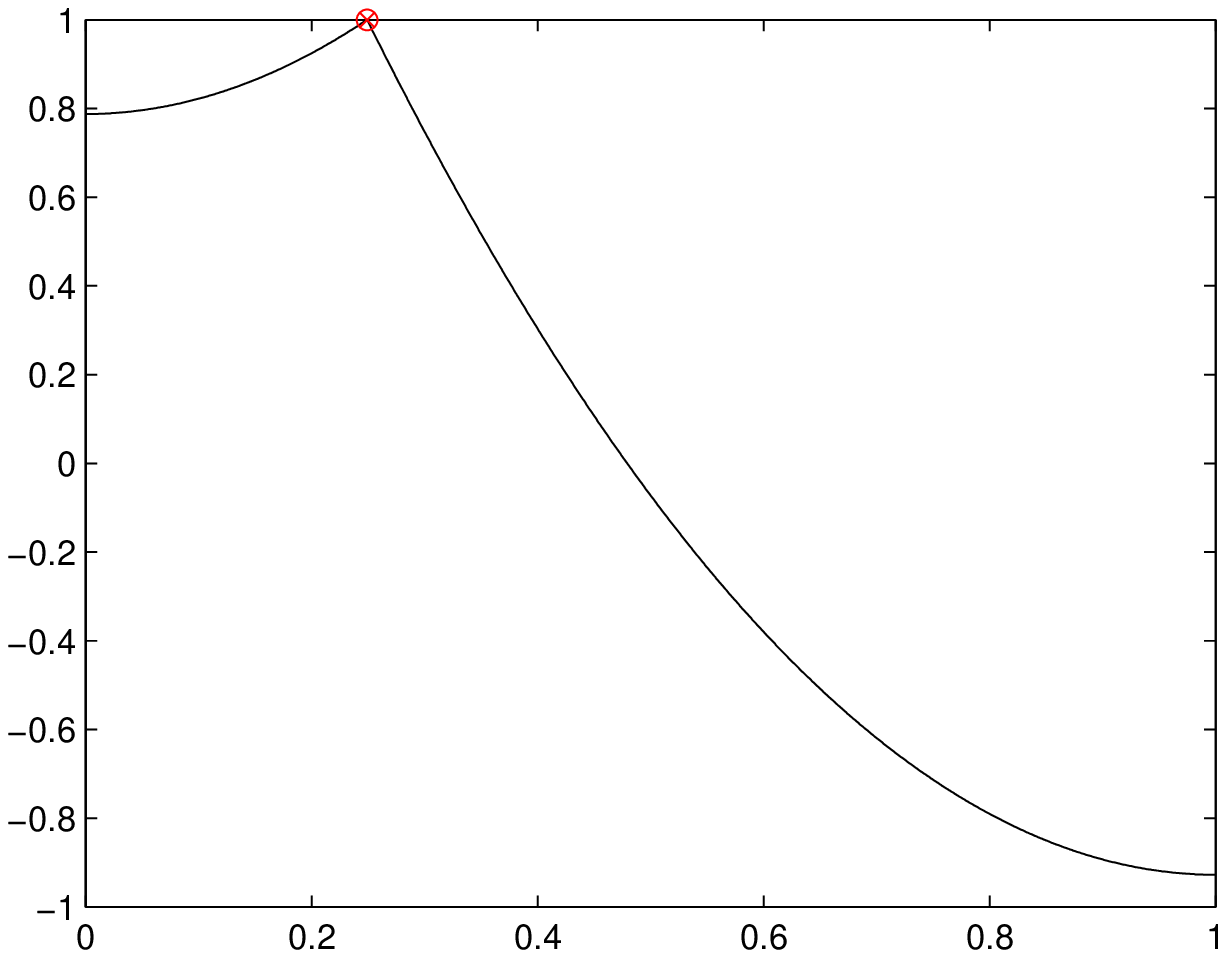}
}\subfigure[Kernel $K^\prime$]{
\includegraphics[width=0.3\columnwidth]{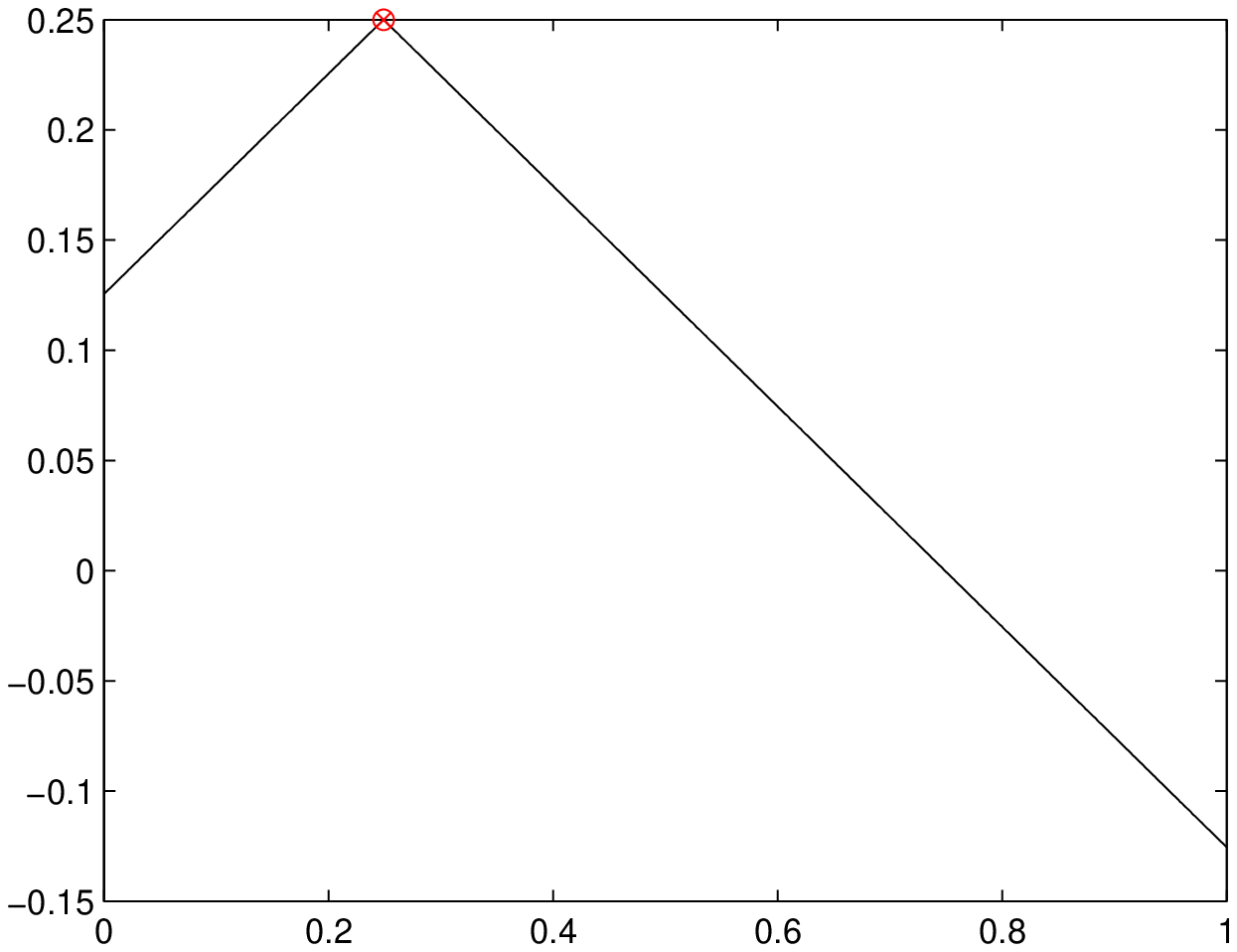}}
\caption{Kernels at $0.25$ over $[0,1]$ in subspace that is
orthogonal to the null of $L_{\alpha,t,n}^u$.} \label{fig:kernel:R1}
\end{center}
\vskip -0.2in
\end{figure}

\bibliographystyle{mlapa}
\bibliography{ref}

\end{document}